\documentclass[sigconf]{acmart}

\usepackage[ruled,vlined,linesnumbered]{algorithm2e}
\usepackage[english]{babel}
\usepackage{amsthm}
\usepackage{multirow}
\usepackage{subcaption}
\usepackage{booktabs}       
\usepackage{amsfonts}       
\usepackage{nicefrac}       
\usepackage{microtype}      
\usepackage{xcolor}         
\usepackage{balance}
\usepackage[inkscapelatex=false]{svg}

\newtheorem{theorem}{Theorem}[section]

\theoremstyle{definition}
\newtheorem{definition}{Definition}[section]

\AtBeginDocument{%
  }

\copyrightyear{2025}
\acmYear{2025}
\setcopyright{acmlicensed}\acmConference[KDD '25]{Proceedings of the 31st ACM SIGKDD Conference on Knowledge Discovery and Data Mining V.1}{August 3--7, 2025}{Toronto, ON, Canada}
\acmBooktitle{Proceedings of the 31st ACM SIGKDD Conference on Knowledge Discovery and Data Mining V.1 (KDD '25), August 3--7, 2025, Toronto, ON, Canada}
\acmDOI{10.1145/3690624.3709316}
\acmISBN{979-8-4007-1245-6/25/08}

\begin{document}

\title{Input Snapshots Fusion for Scalable Discrete-Time Dynamic Graph Neural Networks}

\author{QingGuo Qi}
\email{qiqingguo@zju.edu.cn}
\orcid{0009-0001-0590-0206}
\affiliation{%
  \institution{Zhejiang University} 
  \city{Hangzhou}
  \state{Zhejiang}
  \country{China}
}
\author{Hongyang Chen}
\orcid{0000-0002-7626-0162}
\authornotemark[1]
\email{dr.h.chen@ieee.org}
\affiliation{%
  \institution{Zhejiang Lab}
  \city{Hangzhou}
  \state{Zhejiang}
  \country{China}
}
\author{Minhao Cheng}
\orcid{0000-0003-3965-4215}
\email{mmc7149@psu.edu}
\affiliation{%
  \institution{Pennsylvania State University}
  \city{University Park}
  \state{PA}
  \country{USA}
}
\author{Han Liu}
\orcid{0000-0001-6921-2050}
\email{liu.han.dut@gmail.com}
\affiliation{%
  \institution{Dalian University of Technology}
  \city{Dalian}
  \state{Liaoning}
  \country{China}
}

\renewcommand{\shortauthors}{QingGuo Qi, Hongyang Chen, Minhao Cheng, and Han Liu}

\begin{abstract}

In recent years, there has been a surge in research on dynamic graph representation learning, primarily focusing on modeling the evolution of temporal-spatial patterns in real-world applications. However, within the domain of discrete-time dynamic graphs, the exploration of temporal edges remains underexplored. Existing approaches often rely on additional sequential models to capture dynamics, leading to high computational and memory costs, particularly for large-scale graphs. To address this limitation, we propose the Input {\bf S}napshots {\bf F}usion based {\bf Dy}namic {\bf G}raph Neural Network (SFDyG), which combines Hawkes processes with graph neural networks to capture temporal and structural patterns in dynamic graphs effectively. By fusing multiple snapshots into a single temporal graph, SFDyG decouples computational complexity from the number of snapshots, enabling efficient full-batch and mini-batch training. Experimental evaluations on eight diverse dynamic graph datasets for future link prediction tasks demonstrate that SFDyG consistently outperforms existing methods.
\end{abstract}


\received{29 January 2024}
\received[revised]{10 August 2024}
\received[accepted]{17 November 2024}

\maketitle
\section{Introduction} 
Graphs, versatile data structures comprised of vertices and edges, play a pivotal role in various real-world applications such as social networks \cite{Wu_Lian_Xu_Wu_Chen_2020, Fan_Ma_Li_He_Zhao_Tang_Yin_2019}, molecule graphs \cite{reiser2022graph} and traffic networks \cite{yu2018spatio, li2021spatial}. 
The rise of deep learning has elevated graph neural networks (GNNs) as essential tools for modeling graphs, enabling the depiction of intricate relationships and interactions among vertices via low-dimensional embeddings \cite{kipf2017semi, gilmer2017neural, velikovi2017graph}.
Currently, research in scalable graph representation learning has made significant advancements that can be applied to large-scale graphs with billions of nodes \cite{hamilton2017inductive, pmlr-v97-wu19e, chiang2019cluster}.
However, existing research primarily focuses on static graphs defined by fixed nodes and edges, thereby posing a challenge when transitioning to dynamic graphs. 

Real-world graphs often undergo dynamic changes, with the graph structure evolving over time \cite{egcn20}. 
The basic changes may take place in both the quantity and attributes of nodes and edges in the temporal dimension.
In this work, to streamline our investigation and align with existing literature \cite{egcn20, sankar2020dysat, yang2021discrete, You22, zhang2023dyted, Zhu23, tgn_icml_grl2020, Zuo2018, congwe2023}, we take the assumption that nodes remain fixed and concentrate on the temporal edge events.
Based on the granularity of time step \cite{kazemi2020representation, Zheng2024}, dynamic graphs can be categorized into Continuous-time Dynamic Graphs (CTDGs) and Discrete-time Dynamic Graphs (DTDGs).
In this case, CTDGs aim to predict upcoming events at the subsequent timestamp (or within a short period captured in a mini-batch), so they represent the evolution of dynamic graphs as a stream of edge events \cite{trivedi2019dyrep, kumar2019predicting, tgat_iclr20, tgn_icml_grl2020, wang2021apan}.
On the other hand, the goal of DTDGs is to forecast upcoming events within a defined time frame (e.g., a day, a week, or a month). This involves representing the dynamic graph through a chronological sequence of snapshots, where each snapshot encompasses all events occurring during that time interval.
In this study, our primary focus lies on DTDGs for their ability to provide a comprehensive whole-graph perspective on dynamic graphs, which renders them more susceptible to scalability challenges.

The prevailing research methodology of DTDGs involves utilizing static GNN models \cite{hamilton2017inductive, velikovi2017graph} to represent each graph snapshot separately. Subsequently, a sequential encoder \cite{hochreiter1997long, chung2014empirical, vaswani2017attention} is employed to capture temporal dynamics for predicting all events in the subsequent snapshot \cite{sankar2020dysat, egcn20, yang2021discrete, You22, zhang2023dyted, Zhu23}.
For example, as shown in Figure~\ref{fig:dgnn_lp:a}, in order to forecast potential trades for the subsequent day using data from the previous week, cryptocurrency exchanges could utilize GNN on the daily snapshots to generate five embeddings for each user. These embeddings are then fed into a Transformer \cite{vaswani2017attention} to derive the user's final embedding for prediction.

\begin{figure}
  \centering
  \subfloat[a][Common Setting]{\includegraphics[width=0.95\columnwidth]{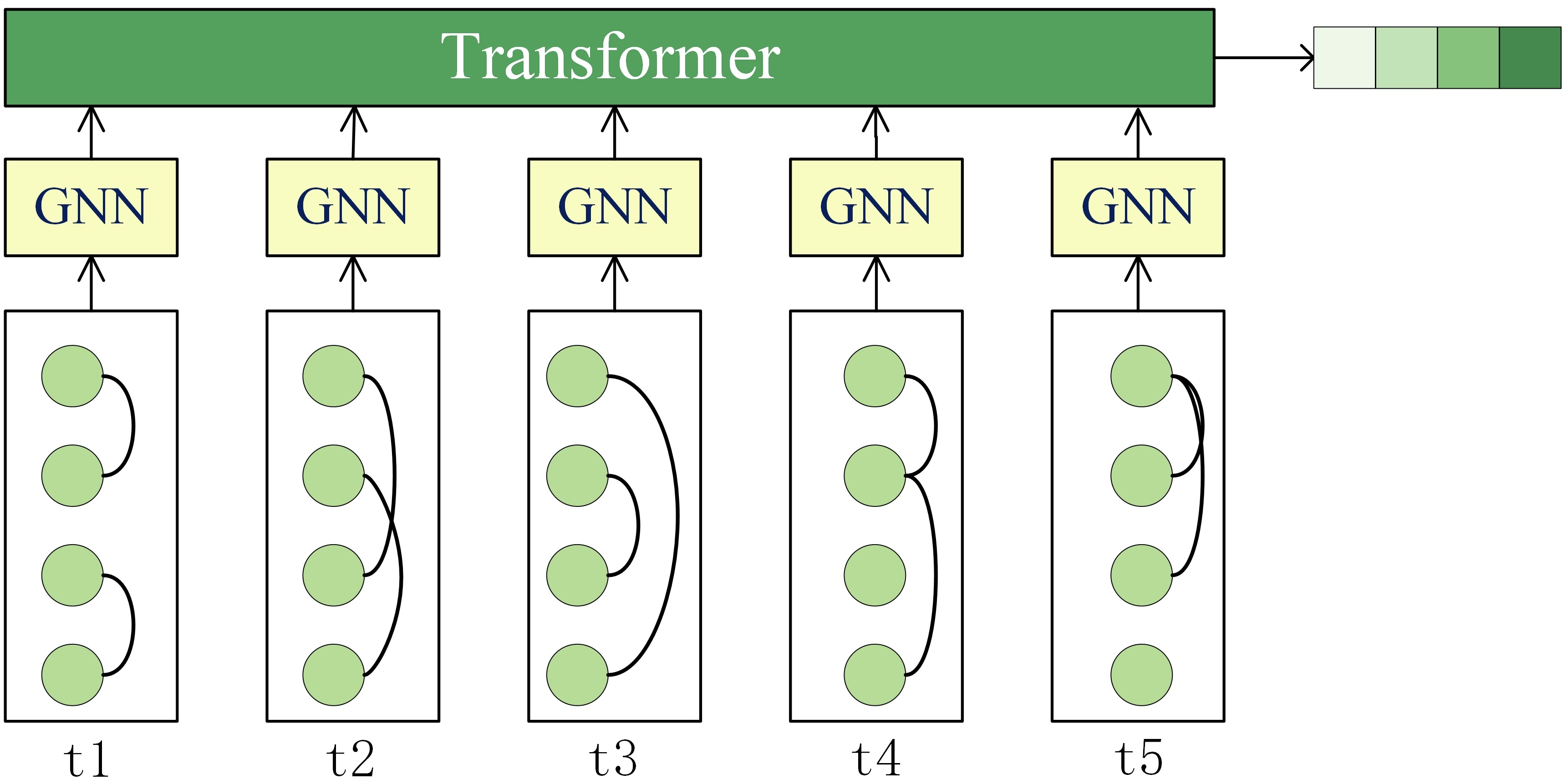} \label{fig:dgnn_lp:a}} \\
  \subfloat[b][Input Snapshots Fusion]{\includegraphics[width=0.95\columnwidth]{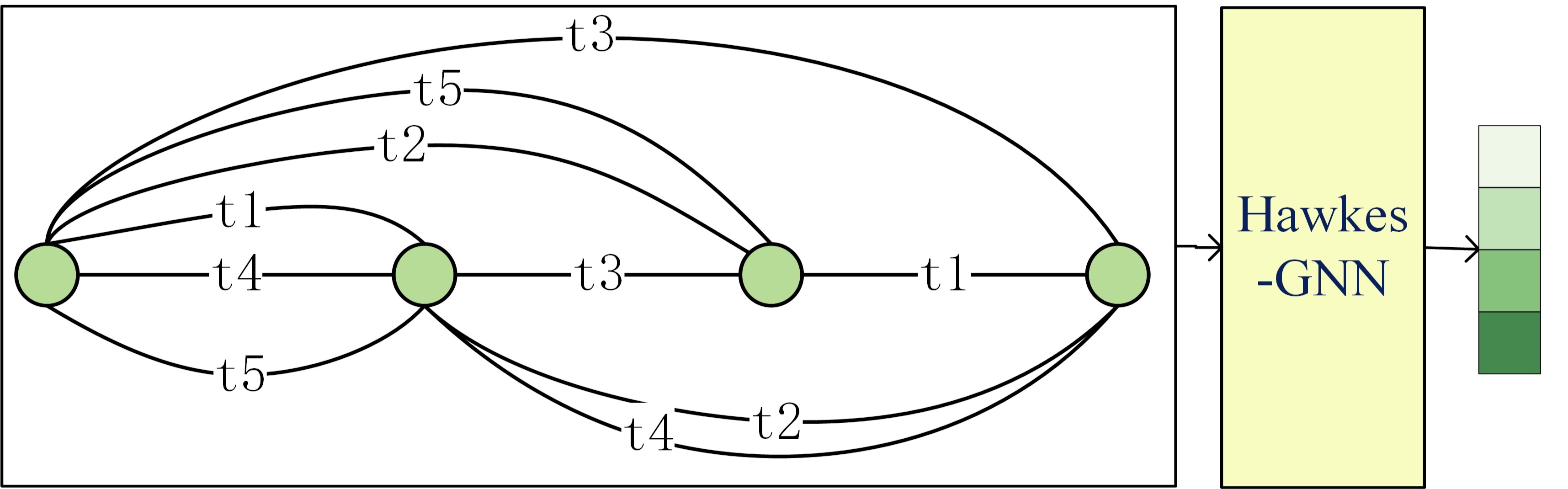} \label{fig:dgnn_lp:b}}
  \caption{Training setting for the task of link prediction in dynamic graphs.} \label{fig:dgnn_lp}
  \Description{The sliding window method has been employed in previous studies as well as in our research for predicting links in dynamic graphs.}
\end{figure}

Simple as it is, this combination represents a compromise since \textit{GNNs designed for static graphs are insufficient for handling multiple temporal edges}.
Furthermore, applying a static GNN to process each snapshot comes with notable limitations.
Precise timestamps are lost, resulting in equal treatment of all edges within a given timeframe, and multiple edges between two nodes are merged into one if they coexist.
Moreover, incorporating sequential encoders increases system complexity as the number of input snapshots grows, which limits the model's scalability on large datasets with long sequences, even when using mini-batch training methods such as GraphSAGE \cite{hamilton2017inductive} or ClusterGCN \cite{chiang2019cluster}.
Therefore, existing studies on the scalability of DGNNs, such as Roland \cite{You22} and Efficient-DGNN \cite{10.1145/3458817.3480858}, often prioritize approaches that involve dividing the sequence.
In this study, we take a reverse approach and propose a novel method termed "input snapshots fusion", as shown in Figure~\ref{fig:dgnn_lp:b}, which merges all input snapshots into a large temporal graph, enabling the presence of multiple temporal edges connecting two nodes.

In the temporal graph, every temporal edge is associated with a timestamp, denoting either a precise time point or a date, depending on the dataset.
To represent temporal edges, we utilize the theory of Hawkes processes to capture events within the adjacency matrix. 
Interestingly, by induction, modeling on the temporal graph can be approximately viewed as a graph denoising problem under the smoothing assumption with time decay.
In other words, the influence of an event relies on its age, with newer events carrying a stronger impact, as observed in applications like communication and recommendation systems.
This formulation leads to the development of a message passing mechanism with time decay, which leads to the integration of Hawkes processes with Graph Convolutional Networks \cite{kipf2017semi} (Hawkes-GCN) and Graph Attention Networks \cite{velikovi2017graph} (Hawkes-GAT).
These novel models allow us to effectively incorporate temporal information into GNN theory for modeling temporal graphs without loss of temporal edges, thereby enhancing expressiveness.
Moreover, by utilizing a single temporal graph as input, both full-batch and mini-batch training methods are independent of the number of snapshots, enabling more effective scalability for large dynamic graphs.

Extensive experiments conducted on eight widely used public datasets demonstrate that our approach outperforms state-of-the-art baselines significantly in link prediction tasks. The adoption of the mini-batch training approach substantially reduces the demand for GPU memory, achieving a maximum reduction of 44\% compared with full-batch training. Furthermore, ablation studies validate the effectiveness of the proposed Hawkes processes-based GNN in accurately modeling temporal graphs, resulting in a potential enhancement of up to 6 times compared to plain GAT. 

In summary, the main contributions are as follows:
\begin{itemize}
    \item We propose a novel idea of input snapshot fusion for discrete-time dynamic graphs, which merges multiple input snapshots into a single temporal graph structure. This approach facilitates a deeper understanding of the temporal evolution of dynamic graphs.
    \item We formulate the graph denoising problem using Hawkes processes theory under the assumption of time-decay smoothing in temporal graphs. This framework enables the design of efficient GNNs that incorporate a time-decay-based message-passing mechanism for dynamic graphs.
    \item We conduct extensive experiments to validate the effectiveness of our proposed model. The experiments demonstrate the superior performance of our approach in link prediction tasks on eight widely used public datasets. 
\end{itemize}

\section{Preliminaries}
In this section, we introduce notations for DTDGs and briefly summarize several important models. 

A discrete-time dynamic graph is defined as a series of snapshots $\{{G}^1, {G}^2, \cdots,{G}^T\}$, where $T$ is the total number of snapshots.
The snapshot at time $t$, i.e., ${G}^t = ({V}^t, {E}^t)$ is a graph with a node set ${V}^t$ and an edge set ${E}^t \subseteq {V}^t \times {V}^t$, where ${E}^{t}_{ij}$ represents an edge from node $i$ to node $j$ at snapshot $t$.
We use $\mathbf{A}^t$ to denote the binary adjacency matrix corresponding to the edge set ${E}^t$, and the neighbors of node $i$ in ${G^t}$, denoted as $N(i)$, is a set of edges, $\{ (i, j) \in {E}^t \}$. 
The superscript "t" can be omitted when it is not specified.

As shown in Figure~\ref{fig:dgnn_lp:b}, a temporal graph $\mathcal{G} = (\mathcal{V}, \mathcal{E}, \mathcal{T})$ was obtained by Input Snapshots Fusion which merges the snapshots.
Specifically, the node set $\mathcal{V}$ is equal to $V$, while the edge set $\mathcal{E}$ represents the union of edges across the snapshots, with $\mathcal{T}$ are the timestamps of the snapshot associated with $\mathcal{E}$.
It is important to note that, in practice, using the exact occurrence time of each edge as  $\mathcal{T}$, when available, rather than the snapshot indices, may improve performance. However, this distinction does not affect our conclusion.
To locate the temporal edges, the neighbors of a node $i$ in $\mathcal{G}$, denoted as ${\mathcal{N}}(i) = \{ (i, j,\tau) \in {(\mathcal{E}, \mathcal{T})} \}$, form a set of edges with timestamp that represent events involving nodes $i$ and $j$ at time $\tau$.

In this paper, we denote the output of the $l$-th layer output of the GNN as ${H}^{l} \in \mathbb{R}^{n \times d}$, where ${H}^{0} = {X}$ represents the input node features. We assume that the input node features remain fixed across different snapshots. Here $n=|\mathcal{V}|$ denotes the number of nodes, $d$ is the embedding dimension, and ${H}_i$ refers to the $i$-th row of ${H}$. In the following, we outline several key models.

\subsection{Link Prediction Tasks in DTDGs}

A discrete-time dynamic graph consists of a sequence of snapshots, which are divided into training, validation, and test sets according to their sequential indices. As shown in Figure~\ref{fig:dgnn_lp:a}, the model input is a sliding window of length $w$, where $w \ge 1$, a configurable hyper-parameter. The model employs spatial and temporal encoders to generate low-dimensional embeddings for each vertex. To predict edges in the subsequent snapshot, a multilayer perceptron (MLP) is used as a predictor, which takes the embeddings of two nodes as input and outputs the probability of a connection in the next time frame.

\subsection{Graph Convolutional Networks}
A single GCN layer~\cite{kipf2017semi} can be written as follows:
\[
\label{eq:gcn_overview}
{H}^{l+1} = \hat{A}{H}^l{W},
\]
where ${W}\in \mathbb{R}^{d \times d}$ is a feature transformation matrix, and $\hat{A}$ is a normalized adjacency matrix.
Additionally, from the perspective of message passing \cite{gilmer2017neural}, the update formula for node i can be expressed as
\begin{equation}
\label{eq:gcn_micro}
{H}_i^{(l+1)} = \sum_{j \in {N}(i) \cup \{ i \}} \frac{1}{\sqrt{\deg(i)} \cdot \sqrt{\deg(j)}} \cdot \left( {W}^{\top}{H}_j^{(l)} \right),
\end{equation}
where the features of neighboring nodes are first transformed by the weight matrix ${W}$, normalized by degree, and then summed.

\subsection{Graph Attention Networks}
Graph Attention Networks (GAT)~\cite{velikovi2017graph} adopts the same message passing mechanism as GCN. The feature aggregation operation for node $i$ is defined as: 
\begin{align}
    { H}_i^{l+1} = \sum\limits_{j\in {{N}}(i)}\alpha_{ij} { H}^l_j,  \quad \text{with} \quad \alpha_{ij}=\frac{\exp \left( e_{ij}\right)}{\sum\limits_{k \in {{N}}(i)} \exp \left(e_{ik}\right)}. \label{eq:gat_propagation}
\end{align}
In this aggregation operation, $\alpha_{ij}$ represents the attention score that differentiates the importance of distinct nodes within the neighborhood. Specifically, $\alpha_{ij}$ is the normalized form of $e_{ij}$, which is defined as:
\begin{align}
\label{eq:att_e_ij}
e_{ij}=\text { LeakyReLU }\left(\left[{ H}'_i \| {H}'_j \right] {\mathbf{a}}\right),
\end{align}
where $[\cdot \|\cdot ]$ denotes the concatenation operation and ${\bf a}\in \mathbb{R}^{2d}$ is a learnable vector.

\subsection{GNNs as Solving Graph Denoising Problem}
According to the conclusion in \cite{ma2021unified}, GNN models can be regarded as an approximation for solving a graph denoising problem under the assumption of smoothness. Formally, Given a noisy input signal ${\bf S}\in \mathbb{R}^{N\times d}$ on a graph ${G}$, the goals is to recover a clean signal ${\bf F}\in \mathbb{R}^{N\times d}$, assumed to be smooth over ${G}$, by solving the following optimization problem: 
\begin{equation}
    \label{eq:denoising}
    \mathop{\arg\min}\limits_{\mathbf{F}} \mathcal{L} = ||\mathbf{F} - \mathbf{S}||_{F}^{2} + \lambda \cdot tr(\mathbf{F}^{\top}\mathbf{L}\mathbf{F}).
\end{equation}
Where the first term guides ${\bf F}$ to be close to ${\bf S}$, while the second term $tr({\bf F}^{\top} {\bf L} {\bf F})$ is the Laplacian regularization which guides ${\bf F}$'s smoothness over ${G}$, with $\lambda>0$'s mediation.

\subsection{Hawkes Process}
The Hawkes process is a typical temporal point process \cite{Hawkes1971SpectraOS} that describes a sequence of discrete events by assuming that previous events have an impact on the current, with the influence diminishing over time. A univariate Hawkes process is defined to be a self-exciting temporal point process whose conditional intensity function $\lambda(t)$ is defined to be
\begin{equation}
  \label{equ:hawkes}
  \lambda(t)=\mu(t) + \sum\limits_{\tau_{i} < t'}\kappa(t'-\tau_i),
\end{equation}
where $\mu(t)$ represents the background intensity, $\tau_{i}$ are the points in time occurring prior to time $t'$; $\kappa$ is an exciting function that models the time decay effect of history on the current event, which is usually in the form of an exponential function, $\kappa(t-s) = exp(-\delta(t-s))$, where the $\delta$ is a non-negative source dependent parameter, representing the time sensitivity of the process.

With the advancement of graph representation learning, Hawkes processes have been incorporated into dynamic graph representation learning to model the influence of historical behaviors on current actions \cite{Zuo2018} and capture the dynamic nature of graph structures \cite{10.1145/3357384.3357943}. However, existing research primarily focuses on integrating Hawkes processes with GNNs within the local scope between node pairs in CTDG. In this paper, we will present how Hawkes processes are effectively integrated with GNNs from a global perspective within the adjacency matrix in the context of DTDG.

\section{The SFDyG Model}
\label{sec:method}
In this section, we present the proposed SFDyG framework. 
First, we examine the feasibility and advantages of input snapshot fusion for DTDGs. 
Next, we explain why the Hawkes processes can enhance graph neural networks (GNNs) in handling temporal edges, demonstrating that modeling temporal graphs can be approximately viewed as a graph denoising problem with time-decayed edge weights. 
Furthermore, we illustrate how existing mini-batch training methods can be seamlessly applied to Hawkes-GNNs for link prediction tasks, enabling more effective scaling to large datasets.

\subsection{Input Snapshots Fusion}
\label{subsec:isf}
As previously discussed, snapshots constitute the foundational structure of DTDGs. However, modeling snapshots separately is inefficient and introduces the risk of losing temporal information.
Drawing from these insights, as illustrated in Figure~\ref{fig:dgnn_lp:b}, we propose to fuse the input snapshots within the sliding window into a single temporal graph to improve the modeling of DTDG. Formally,
\begin{equation}
\label{equ:merge}
    {\mathcal{G}} = \text{Fusion}({G}^1,...,{G}^{w}) = ({V}, {E}^1\cup\cdots\cup{E}^{w}),
\end{equation}
where ${G}^1,...,{G}^{w}$ are the input snapshots in the sliding window and $\cup$ represents union of two edge sets. 

In the generated temporal graph, each temporal edge is associated with a time attribute, which may be an exact timestamp or a snapshot index, depending on the dataset. It is ensured that the sequence of events between any two nodes is ordered, and the temporal granularity across all events remains consistent.
In this way, it is possible to more effectively utilize time information with long dependencies, as well as apply state-of-the-art algorithms and acceleration methods from static graph fields.
However, the generated temporal graph $\mathcal{G}$ is beyond the capabilities of static GNNs, as multiple temporal edges can exist between two nodes.

\subsection{Hawkes Processes Based GNN}
Hawkes processes are powerful tools for modeling sequences of historical events, especially in scenarios where events recur, rendering them well-suited for modeling temporal edges.
However, directly applying Hawkes processes to the temporal edges would result in $O(n^2)$ possible pair of nodes, making this approach infeasible to scale to large datasets.

To address this challenge, we propose to integrate Hawkes processes (Equation~\ref{equ:hawkes}) with the message-passing framework (Equation~\ref{eq:gcn_micro}) to update node embeddings rather than focus on node pairs. 
Specifically, the base rate $\mu(t)$ and self-excitation $\sum\kappa(t'-\tau)$ are captured by the linear layer and the message aggregation layer, respectively.
Finally, to predict the probability of future links of node pairs, their embeddings are concatenated and processed through an MLP predictor.

Next, we connect Hawkes processes-based message passing neural networks with the graph denoising problem by introducing Hawkes excitation matrix $\mathcal{C}$ to gain a deeper understanding and derive formulations of Hawkes processes-based GNNs.

\begin{definition}[Hawkes excitation matrix]
a symmetric matrix $\mathcal{C}$ where each element $\mathcal{C}_{ij}$ represents the excitation effect of previously occurred events from node $i$ to node $j$,
\begin{equation}
    \mathcal{C}_{ij} = \sum_{(i,j,\tau) \in {\mathcal{E}}_{ij}}exp(-\delta(t' - \tau)), ~~ \text{with}~ \tau < t',
\end{equation}
\end{definition}
where $t'$ denotes the starting point of the time frame to be predicted, while the non-negative scalar $\delta$ represents the time sensitivity parameter. 
It is noteworthy that the Hawkes excitation matrix shares the same shape and zero elements as the adjacency matrix, but their non-zero elements differ. In the Hawkes excitation matrix, non-zero elements are expressions that quantify the excitation effect of temporal edges rather than indicating the presence of an edge. Thus, the Hawkes excitation matrix can be regarded as an extension of the adjacency matrix in the context of temporal graphs.

Based on these definitions, the temporal graph denoising problem can be formulated by employing $\mathcal{C}$ in Equation~\ref{eq:denoising}, as described in Theorem~\ref{theorem:hawkes-denoising}.
\begin{theorem}[Hawkes temporal graph denoising]
\label{theorem:hawkes-denoising}
When we adopt the Laplacian matrix ${{\bf L}} = {{\bf D}} - {\mathcal{C}}$, where ${\bf D} = diag(\sum_j{{\mathcal{C}}_{ij}})$, the graph denoising problem (Equaition~\ref{eq:denoising}) was extended to temporal domain with the assumption that the smoothness of edges decays over time.
\end{theorem}
\begin{proof}
\begin{equation}
\label{eq:decay_denoising}
\begin{split}
\lambda \cdot tr(\mathbf{F}^{\top}{\mathbf{L}}\mathbf{F}) &= \lambda \cdot \sum_{(i,j) \in {E}}\mathcal{C}_{ij}||\mathbf{F}_i - \mathbf{F}_j||_2^2  \\
&= \lambda \cdot \sum_{(i,j) \in {E}} \sum_{(i,j,\tau) \in {\mathcal{E}}_{ij}}exp(-\delta(t' - \tau))||\mathbf{F}_i - \mathbf{F}_j||_2^2 \\
&= \lambda \cdot \sum_{(i,j,\tau) \in {\mathcal{E}}} exp(-\delta(t' - \tau))||\mathbf{F}_i - \mathbf{F}_j||_2^2.
\end{split}
\end{equation}
Where each element of the summation is in the form of a time decay coefficient multiplied by a graph smoothing indicator. Which completes the proof.
\end{proof}
Theorem~\ref{theorem:hawkes-denoising} demonstrates that, in Hawkes processes-based temporal graph denoising problem, the influence of a temporal edge on its two endpoints depends on both the age of the edge and the time sensitivity parameter $\delta$. 
In other words, the Hawkes excitation matrix enables the graph denoising method to handle temporal graphs under the assumption that the smoothness constraints imposed by temporal edges diminish over time.

Then, Hawkes processes based GNNs can be derived by applying the findings from \cite{ma2021unified}, which demonstrate that the graph denoising problem can be approximately regarded as a GNN model.

\subsubsection{Hawkes-GCN}
By adopting a normalized Hawkes Laplacian matrix, defined as ${\mathbf{L}} = {D}^{-\frac{1}{2}}({D} - {\mathcal{C}}){D}^{-\frac{1}{2}}$ for Equation~\ref{eq:denoising}, where ${D} = \text{diag}(\sum_j{{{A}}_{ij}})$ is the diagonal degree matrix derived from the binary adjacency matrix, the temporal graph denoising problem is connected to the Hawkes-GCN model \cite{ma2021unified}. Formally,
\begin{equation}
  {H}^{k+1} = {D}^{-\frac{1}{2}}{\mathcal{C}}{D}^{-\frac{1}{2}}{ H}^{k}{ W},
\end{equation}
specifically, the $i$-th element in the node embedding ${H}$ can be expressed as follows,
\begin{equation}
\label{eq:hawkes_gcn}
\begin{split}
    {H}_i^{(k+1)} &= \sum_{j \in {N}(i)} \frac{\mathcal{C}_{ij}}{\sqrt{deg(i)deg(j)}} \cdot {W}^{\top}{H}_j^{(k)} \\
    &= \sum_{(i,j,\tau) \in {\mathcal{N}}(i)}\frac{exp(-\delta_i(t' - \tau))}{\sqrt{deg(i)deg(j)}}{W}^{\top}{H}_j^{(k)}.
\end{split}
\end{equation}
Where $\delta_i$ is a learnable parameter indicating that the time sensitivity parameter depends on the source node $i$. 
Through Equation \ref{eq:hawkes_gcn}, it can be observed that, during the message aggregation process in the temporal graph, each temporal edge transmits messages whose influence diminishes over time. Accordingly, we denote this Hawkes processes-based GNN as a time-decayed message-passing neural network.
In addition, we use the diagonal matrix derived from the binary adjacency matrix for regularization to incorporate the information of the unique number of neighbors for each node, which generally results in better performance. However, for dense temporal graphs, we recommend applying batch normalization \cite{ioffe2015batch} to mitigate potential numerical instabilities.

\subsubsection{Hawkes-GAT}
The success of GAT is to calculate the non-negative attention score $\alpha_{ij}$ to differentiate the importance of distinct nodes in the neighborhood, which is a natural choice of the time sensitivity coefficient $\delta$ in Equation~\ref{eq:hawkes_gcn}. Then, the GAT model is extended to the temporal domain, Formally,
\begin{equation}
\begin{split}
    &{H}_i^{(k+1)} = \sum_{(i,j,\tau) \in {\mathcal{N}}(i)}\frac{exp(-\delta_{ij}(t'-\tau))}{\sqrt{deg(i)deg(j)}}{W}^\top{H}_j^{(k)}, \\
    &\text{with}~~~\delta_{ij} = \frac{exp(e_{ij})}{\sum_{k \in {N}(i)}exp(e_{ik})}.
\end{split}
\end{equation}
Where $e_{ij}$ is described in Equation~\ref{eq:att_e_ij}. In this way, the Hawkes-GAT is derived under the assumption that the time sensitivity depends on both the source and target node of an edge, enabling more flexible modeling.

\subsection{Loss Function}
The link prediction task for the discrete-time dynamic graphs involves two steps \cite{egcn20, sankar2020dysat, You22}: first, generating all node embeddings $\mathit{H}$ using an encoder, denoted as $\textsc{Gnn}$. Second, predicting the target link probability based on the embeddings of its two endpoints, using an MLP referred to as $\textsc{Fc}$. The forward propagation process can be formalized as:
\begin{equation}
\label{eq:enc_dec}
    \mathit{H} = \textsc{Gnn}(G_{t-w},...,G_{t-1}),
    \hat{y}_{ij}= \sigma(\textsc{Fc}(\mathit{H}_i, \mathit{H}_j)).
\end{equation}

where ${H}_i$ represents the embedding of node $i$, and $\hat{y}_{ij}$ denotes the likelihood of a future connection from node $i$ to node $j$. The benefit of this two steps model is that node embeddings can be reused, which is especially suitable for training with multiple negative samples. Consistent with existing research, the cross-entropy loss is used as the loss function:
\begin{equation}
    \mathcal{L} = \frac{1}{|{E}|}\sum_{e_{ij} \in {E}}{(-y_{ij} \cdot log(\hat{y}_{ij}) - (1 - y_{ij}) \cdot log(1 - \hat{y}_{ij}))}.
\end{equation}
Where the term $E$ represents the collection of all positive edges union all randomly sampled negative edges.

\subsection{Minibatch Training Algorithm}
In this subsection, we demonstrate that the existing mini-batch algorithm can be efficiently applied to our proposed Hawkes processes based GNNs for link prediction tasks.

While training efficiency is improved by decoupling from window length through Input Snapshots Fusion, negative sampling plays a critical role in redundant computations. To further enhance efficiency, we adopt single negative sampling, as demonstrated in our experiments in section~\ref{sec:exp}, which confirms its sufficiency.
Additionally, the link neighbor loader in PyG \cite{Fey/Lenssen/2019} is employed for mini-batching temporal graphs, enabling a straightforward mini-batch algorithm, as detailed in Algorithm~\ref{algo:minibatch}.

\begin{algorithm}
	\caption{One Epoch of Mini-batch Training for SFDyG} \label{algo:minibatch}
	\SetAlgoLined
    \SetKwComment{Comment}{/* }{ */}
	\KwIn{Encoder $\textsc{Gnn}$, decoder $\textsc{Fc}$, window length $w$ and the dynamic graph $\{{G}^1, {G}^2, \cdots,{G}^T\}$ }
	\KwOut{Updated $\textsc{Gnn}$, \textsc{Fc}}
	
    \For{\texttt{index} i \texttt{from} w+1 \texttt{to} T}{
        $\mathcal{G} \leftarrow \texttt{Fusion}(G^{i-w},\cdots,G^{i-1}$) \;
        ${loader} \leftarrow \texttt{LinkNeighborLoader}(\mathcal{G}, G^i.Edge)$ \;
        \For{\texttt{batch} in $loader$}{ 
            $\texttt{h} \leftarrow \textsc{Gnn}(\texttt{batch})$ \;
            $\hat{\texttt{y}} \leftarrow \textsc{Fc}(\texttt{h}, \texttt{batch.edge\_label\_index})$ \;
            $\texttt{y} \leftarrow \texttt{batch.edge\_label} $ \;
            $\texttt{loss} \leftarrow \textit{cross-entropy}(\hat{\texttt{y}}, \texttt{y})$ \;
            $\texttt{loss}.\texttt{backward}()$ \;
            $\texttt{optimizer}.\texttt{step}()$ \;
        }
    }
\end{algorithm}
Compared to other models, mini-batching with Hawkes-GNN is more memory efficient as it stores the node embeddings $h$ only once. Additionally, it runs faster as there is no additional temporal encoder, which makes it more scalable to large datasets.
It is important to emphasize that Algorithm \ref{algo:minibatch} differs significantly from the mini-batch training approach used in CTDG \cite{tgn_icml_grl2020}. Specifically, it processes the dynamic graph from a whole graph perspective, encompassing nodes over a large time frame, rather than focusing on a subgraph within a small time interval. Moreover, the steps can be executed in parallel rather than sequentially in chronological order, making it more efficient with parallel computation.

\subsection{Complexity Analysis}
\begin{table}[]
\centering
\caption{Comparison of time and memory complexities.}
\label{tab:complexity}
\begin{tabular}{c|cc}
\hline
Method       & \multicolumn{1}{c}{Time} & \multicolumn{1}{c}{Memory} \\ \hline
DySAT        & $O(tlef + tlnf^2)$    & $O(tlnf + tlf^2)$  \\
EvolveGCN    & $O(tlef + tlnf^2)$    & $O(tlnf + tlf^2)$  \\
Roland       & $O(lef + lnf^2)$      & $O(lef + lf^2)$    \\
SFDyG-F      & $O(tlef + lnf^2)$     & $O(tlef + lf^2)$   \\ 
SFDyG-M      & $O(2ted^lf^2)$            & $O(2bd^lf + lf^2)$   \\ \hline
\end{tabular}
\vspace{-0.1in}
\end{table}

In this subsection, time and space complexity analyses were provided for SFDyG and the following representative DGNNs: DySAT \cite{sankar2020dysat}, EvolveGCN \cite{egcn20} and ROLAND \cite{You22}.

In the context of discrete-time dynamic graph link prediction, we consider several parameters to analyze the complexity of the algorithms. The snapshots are arranged in a slice window with length $t$, and the total number of nodes is represented as $n=|V|$. Moreover, we have the average number of edges per snapshot denoted as $e=|E|$, and the average degree per node as $d$. Without loss of generality, we assume that the node feature dimension and the length of the hidden vectors in the network are both $f$. Additionally, the number of layers in the graph neural network is represented as $l$, and the batch size for mini-batch training is denoted as $b$. The time and memory complexities are summarized in Table~\ref{tab:complexity}, while detailed analyses are provided in Appendix~\ref{sec:complexity}.

Overall, the full-batch version denoted as SFDyG-F exhibits the second-best time and space complexity, only surpassed by Roland. In comparison to Roland, SFDyG demonstrates enhanced capability in capturing long-range temporal dependencies, without being constrained by sequential training in chronological order. Typically we have $l <= 3$ for GNN and large graphs tend to exhibit sparsity, therefore, the mini-batch version denoted as SFDyG-M showcases superior space complexity, making it suitable for application in large-scale dynamic graphs with numerous nodes.

\section{Experiments}
\label{sec:exp}
\begin{table*}
  \caption{Overall performance (MRR@100) comparison on eight datasets (\% is omitted). Our experiments guarantee consistent data settings and standardized methods for computing Mean Reciprocal Ranks (MRRs) to facilitate fair comparisons. Each experiment is conducted using three random seeds, and the average performance is reported along with the standard error.}
  \label{tab:main_mrr}
  \begin{tabular}{cccccccccl}
    \toprule
    Methods & OTC & Alpha & UCI & Title & Body & AS733 & SBM & SO \\
    \midrule
\texttt{DySAT}   &21.39 $\pm$ 2.79 & 19.16 $\pm$ 2.21 & 23.31 $\pm$ 9.42 & 17.46 $\pm$ 4.18 & 13.87 $\pm$ 3.90 & 25.10 $\pm$ 1.71 & 6.88 $\pm$ 0.53 & OOM.    \\
\texttt{EvolveGCN}   &7.84 $\pm$ 0.09 & 6.65 $\pm$ 0.55 & 7.33 $\pm$ 0.15 & 30.67 $\pm$ 0.00 & 18.55 $\pm$ 0.02 & 42.06 $\pm$ 0.00 & 21.38 $\pm$ 0.00 & 31.21 $\pm$ 0.48    \\
\texttt{Roland}   &30.94 $\pm$ 0.70 & 32.97 $\pm$ 1.78 & 17.04 $\pm$ 2.30 & 46.33 $\pm$ 0.27 & \underline{38.57 $\pm$ 0.42} & 21.21 $\pm$ 5.73 & 1.96 $\pm$ 0.00 & 38.57 $\pm$ 1.44    \\
\texttt{WinGNN}   &3.86 $\pm$ 1.26 & 3.90 $\pm$ 0.84 & 2.37 $\pm$ 0.13 & 4.19 $\pm$ 1.25 & 2.69 $\pm$ 0.38 & 4.29 $\pm$ 2.10 & 3.35 $\pm$ 0.50 & 7.51 $\pm$ 0.67    \\
\texttt{VGRNN}   &6.62 $\pm$ 0.10 & 6.49 $\pm$ 0.29 & 6.96 $\pm$ 0.08 & OOM. & 17.19 $\pm$ 0.14 & 41.94 $\pm$ 2.04 & 19.79 $\pm$ 0.23 & OOM.    \\
\texttt{HTGN}   &6.36 $\pm$ 0.06 & 7.72 $\pm$ 0.66 & 8.67 $\pm$ 0.43 & 11.50 $\pm$ 0.98 & 10.70 $\pm$ 0.52 & 13.86 $\pm$ 0.58 & 10.92 $\pm$ 1.19 & OOM.    \\
\texttt{GraphMixer}   &43.67 $\pm$ 0.25 & 35.72 $\pm$ 0.41 & 33.63 $\pm$ 0.02 & 38.32 $\pm$ 0.01 & 33.15 $\pm$ 0.02 & 28.86 $\pm$ 0.00 & 1.96 $\pm$ 0.00 & OOM.    \\
\texttt{M2DNE}   &7.82 $\pm$ 1.05 & 5.49 $\pm$ 0.29 & 8.86 $\pm$ 0.44 & 5.40 $\pm$ 0.05 & 6.03 $\pm$ 0.38 & 19.43 $\pm$ 0.12 & OOM. & OOM.    \\
\texttt{GHP}   &3.40 $\pm$ 0.41 & 3.40 $\pm$ 0.46 & 4.15 $\pm$ 0.14 & 16.00 $\pm$ 2.32 & 8.33 $\pm$ 2.00 & 22.15 $\pm$ 4.88 & 26.85 $\pm$ 17.65 & OOM.    \\ \midrule
\texttt{Hawkes-GCN}   &\underline{46.16 $\pm$ 0.45} & \textbf{47.87 $\pm$ 5.85} & \textbf{35.61 $\pm$ 0.06} & \underline{47.44 $\pm$ 0.20} & 36.44 $\pm$ 0.42 & \underline{44.34 $\pm$ 0.41} & \textbf{29.10 $\pm$ 0.73} & \underline{46.41 $\pm$ 0.31}    \\
\texttt{Hawkes-GAT}   &\textbf{51.34 $\pm$ 0.07} & \underline{40.66 $\pm$ 0.25} & \underline{35.59 $\pm$ 1.58} & \textbf{50.84 $\pm$ 0.05} & \textbf{40.97 $\pm$ 0.47} & \textbf{45.95 $\pm$ 0.79} & \underline{28.96 $\pm$ 0.70} & \textbf{48.83 $\pm$ 0.14}    \\
    \bottomrule
  \end{tabular}
\end{table*}
\subsection{Experimental Setup}
\subsubsection{Datasets}
We conducted experiments on eight commonly used public datasets that have been extensively evaluated in previous studies on dynamic graph representation learning, encompassing Bitcoin-Alpha, Bitcoin-OTC, UCI, Reddit-Title, Reddit-Body, AS733, and Stack Overflow. The fundamental statistics of the eight datasets are presented in Table~\ref{tab:dataset}. Following EvolveGCN \cite{egcn20}, we subdivided the original dataset into multiple snapshots of equal frequency. Subsequently, the training, validation, and test sets are divided along the time dimension. Details of the datasets can be found in Appendix~\ref{sec:ap_ds_dsec}.

\begin{table}[]
\caption{Summary of dataset statistics.}
\label{tab:dataset}
\begin{tabular}{ccccc}
\toprule
& \# Nodes & \# Edges & \# Time Steps & Avg.  \\
&          &          & (Train/Val/Test) & Degree \\
\midrule
UCI      & 1,899        & 59,835     & 35 / 5 / 10 & 0.36\\
Alpha & 3,777        & 24,173     & 95 / 13 / 28 & 0.04\\
OTC   & 5,881        & 35,588     & 95 / 14 / 28 & 0.05\\
Title & 54,075   & 571,927    & 122 / 35 / 17 & 0.06\\
Body  & 35,776   & 286,562    & 122 / 35 / 17 & 0.05\\
AS733       & 7,716        & 1,167,892  & 70 / 10 / 20 & 2.12\\ 
SBM      & 1,000        & 4,870,863  & 35 / 5 / 10 & 97.42\\
SO       & 2,601,997    & 63,497,050 & 65 / 9 / 18 & 0.12\\ 
\bottomrule
\end{tabular}
\end{table}

\subsubsection{Baselines}
We evaluated the performance of our proposed model, Hawkes-GCN, Hawkes-GAT by comparing to several dynamic GNN baselines, namely EvolveGCN \cite{egcn20}, DySat \cite{sankar2020dysat}, GHP \cite{Shang_Sun_2019}, ROLAND \cite{You22}, GraphMixer (G-Mixer for short) \cite{congwe}, M2DNE \cite{10.1145/3357384.3357943}, VGRNN \cite{hajiramezanali2019variational}, HTGN \cite{yang2021discrete} and WinGNN \cite{Zhu23}. 
Note that some baseline models were originally designed for modeling the dynamics of CTDG. To demonstrate the superiority of Hawkes-GNN, we reimplemented these models and adapted them to the DTDG setting.
In Appendix~\ref{sec:ap_baseline}, a comprehensive description of these baselines can be found.

\subsubsection{Evaluation metrics}
We evaluate the effectiveness of the SFDyG framework in the context of future link prediction. Our primary evaluation metric is the Mean Reciprocal Rank (MRR) with 100 negative sampling as defined in OGB \cite{hu2020open}, which is an average of the pessimistic and optimistic ranks. An analysis of different MRRs based on various statistical approaches is presented in Appendix~\ref{sec:ap_metric}. 

\subsubsection{SFDyG Architecture}
SFDyG adopts the prevalent encoder-decoder architecture for future link prediction, featuring a two-layer Hawkes processes based GNN as the encoder to generate embeddings for all nodes. The model utilizes a two-layer MLP as the decoder, taking a pair of nodes as input and determining the probability of their forthcoming connection. To maintain parity in evaluations, all models in the experiment share identical decoders architecture, differing only in their encoders.

\subsection{Main Results}

\subsubsection{Full-batch Training}
The performance of the proposed SFDyG and other baseline models in dynamic link prediction is presented in Table~\ref{tab:main_mrr}. The results reveal significant variations in the effectiveness of existing baseline models across different dynamic graph datasets. 
DySAT and EvolveGCN, equipped with temporal encoders, demonstrate better performance on denser graphs like AS733 and SBM, suggesting possible underutilization of temporal encoders. 
Conversely, models with single snapshot inputs, such as VGRNN and HTGN, exhibit similar performances, indicating potential neglect of temporal features. 
GraphMixer and Roland emerge as the top among the baselines, however, they fail to learn on dense datasets like SBM due to limited history neighbors and completed temporal patterns.
Hawkes processes based methods like M2DNE and GHP behave better on dense datasets, but they are likely to OOM and behave poorly on sparse datasets like OTC and Alpha.
In contrast, our proposed model SFDyG showcases substantial advantages over all datasets, outperforming baseline models by a considerable margin, highlighting the efficacy of our GNN based on Hawkes processes in capturing temporal information in temporal graphs. Generally, Hawkes-GAT demonstrates slightly superior performance compared to Hawkes-GCN, except for datasets with minimal edges like bitcoin-alpha and nodes like SBM, thereby affirming the more adaptable modeling capability of Hawkes-GAT. 
Collectively considering these factors, we assert that SFDyG surpasses existing baseline methods in dynamic link prediction capabilities.

Figure~\ref{fig:gpu} illustrates the utilization of GPU memory by the SFDyG model throughout the training procedure as compared to various standard baseline methods. The findings demonstrate a notable superiority of our approach over other multi-snapshot input baselines, some of which face out-of-memory (OOM) issues when applied to the StackOverflow dataset. GPU memory consumption of our proposed methods closely resembles that of the single snapshot input method Roland, exhibiting a lower constant space complexity, albeit performing less effectively than Roland when applied to the SBM dataset with an edge degree of about 100. The higher efficiency in memory usage of our method suggests promising scalability potential.

\begin{figure}[]
  \centering
  \includegraphics[width=\linewidth]{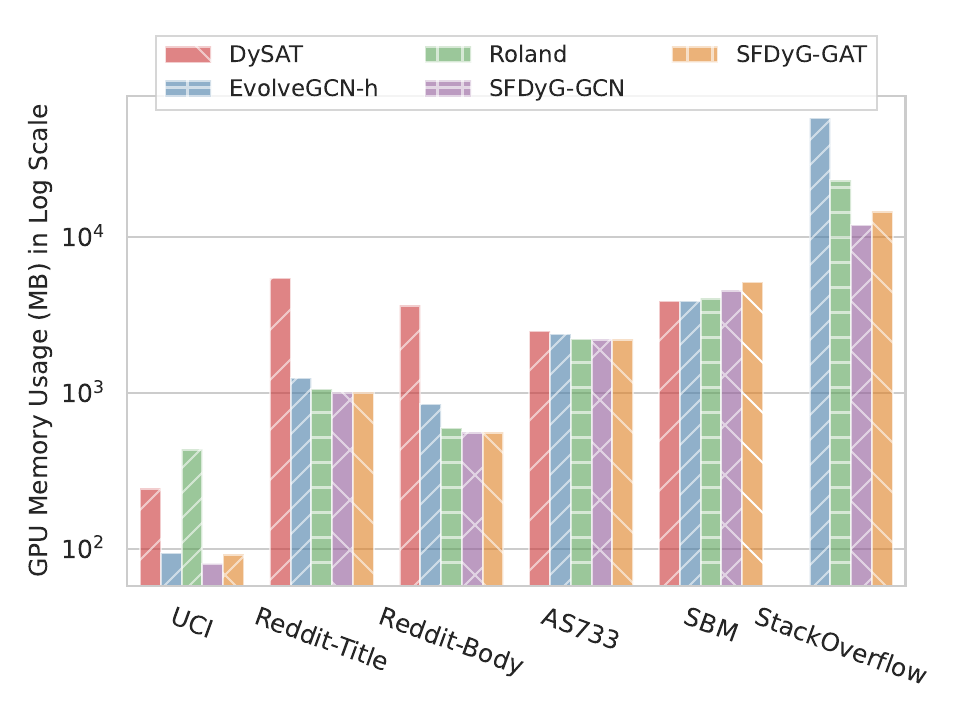}
  \caption{The GPU memory usage of SFDyG and representative baselines on six datasets.}
  \Description{The GPU memory consumption of our method is similar to that of Roland's single snapshot input method, and in most cases, it performs better.}
  \label{fig:gpu}
\end{figure}

\subsubsection{Mini-batch Training vs. Full batch Training}
\begin{table}[]
  \caption{The performance of full-batch training and mini-batch training was compared on eight datasets in terms of the relative percentage change in MRR@100 ($\Delta$ MRR@100) and GPU memory usage ($\Delta$ GPU).}
  \label{tab:minibatch}
  \begin{tabular}{c|cccc}
    \toprule
    {Dataset} & Full-batch &  Mini-batch & $\Delta$Mrr@100 & $\Delta$GPU \\
    \midrule
OTC  & 51.34 $\pm$ 0.07  &  52.59 $\pm$ 1.49  & $\uparrow$ 2.43\%  & $\downarrow$ 4.13\%  \\
Alpha  & 40.66 $\pm$ 0.25  &  40.61 $\pm$ 0.19  & $\downarrow$ 0.12\%  & $\downarrow$ 13.01\%  \\
UCI  & 35.59 $\pm$ 1.58  &  39.00 $\pm$ 0.11  & $\uparrow$ 9.58\%  & $\downarrow$ 0.67\%  \\
Title  & 50.84 $\pm$ 0.05  &  51.17 $\pm$ 0.17  & $\uparrow$ 0.65\%  & $\downarrow$ 6.59\%  \\
Body  & 40.97 $\pm$ 0.47  &  41.87 $\pm$ 0.18  & $\uparrow$ 2.20\%  & $\downarrow$ 7.78\%  \\
AS733  & 45.95 $\pm$ 0.79  &  52.59 $\pm$ 1.18  & $\uparrow$ 14.45\%  & $\downarrow$ 2.03\%  \\
SO  & 48.83 $\pm$ 0.14  &  47.84 $\pm$ 0.05  & $\downarrow$ 2.03\%  & $\downarrow$ 44.01\%  \\
SBM  & 28.96 $\pm$ 0.70  &  29.45 $\pm$ 0.54  & $\uparrow$ 1.69\%  & $\uparrow$ 0.19\%  \\
    \bottomrule
  \end{tabular}
\end{table}

We train our proposed Hawkes-GAT using the mini-batch training algorithm on eight datasets to compare the memory consumption and performance between full-batch and mini-batch training. The batch size is determined according to the size of the dataset, and all neighbors are selected by the neighbor sampler.

The comparative results of full-batch and mini-batch training are presented in Table \ref{tab:minibatch}.  In general, there is a noticeable decrease in GPU consumption in the majority of cases, particularly with large datasets such as StackOverflow, showing a reduction of more than 44\%. However, for smaller and denser datasets like SBM, there is a slight increase of 0.19\% in GPU consumption, suggesting that mini-batch training approaches may not be ideal in these particular situations.
Moreover, the experimental results exhibit minimal variance in most datasets except on UCI and AS733 increased by 9.58\% and 14.45\% respectively. This phenomenon may be attributed to the capacity of small-batch training to mitigate the influence of supernodes.

\subsubsection{Mini-batch Training vs. Mini-batch Training}

\begin{figure}[]
  \centering
  \includegraphics[width=\linewidth]{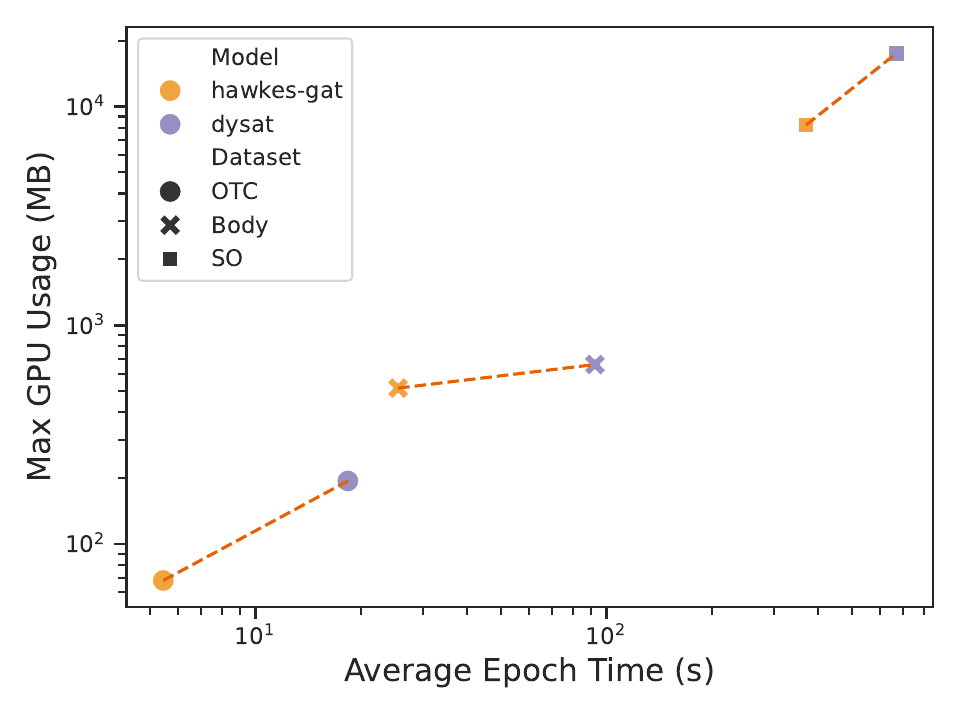}
  \caption{The comparison of average epoch time and maximum GPU memory usage during mini-batch training. The dotted line represents one dataset, while the endpoints illustrate the execution metrics on that dataset.}
  \Description{}
  \label{fig:mini}
\end{figure}

In the existing literature, few studies have explored the application of mini-batch training to DTDGs. A key challenge is aligning selected nodes and their neighbors across snapshots. Unexpectedly, we find that Input Snapshots Fusion helps address this issue. Specifically, for each snapshot, an additional attribute is added to every edge to record the snapshot's index. The mini-batch method is then applied to the fused temporal graph to generate temporal subgraphs, ensuring that all historical neighbors of the selected nodes are sampled. Finally, based on the snapshot index stored in the edge attributes, the temporal subgraph is divided back into multiple sub-snapshots. This approach enables traditional multi-snapshot models to leverage mini-batch training, thereby scaling effectively to large datasets.

Based on the above analysis, we evaluated mini-batch training on DySAT and Hawkes-GAT by measuring the average single-epoch training time and maximum GPU memory usage across datasets of varying sizes, as shown in Figure \ref{fig:mini}. The three datasets are UCI, Body, and SO, representing small, medium, and large datasets, respectively. The results indicate that, under the same batch size, Hawkes-GAT consistently achieves better scalability with faster training speed and lower memory usage.

\subsection{Ablation Study}

In modern graph neural network libraries such as PyG \cite{Fey/Lenssen/2019}, the data structures for static graphs and temporal graphs are identical. This implies that plain graph neural network algorithms, like GAT, can be utilized on temporal graphs formed through input snapshots fusion methods. Hence, the question arises: Is it necessary to develop dedicated algorithms for temporal graphs? To answer this question, we run experiments between the plain GAT and Hawkes-GAT.

\begin{table}
  \caption{The overall performance comparison (MRR@100) between the plain GAT and the Hawkes-GAT, with the relative percentage of improvements.}
  \label{tab:ablation}
  \begin{tabular}{cccc}
    \toprule
    DataSet & GAT & Hawkes-GAT & Improve \\
    \midrule
OTC &  15.18 $\pm$ 7.63  &  51.34 $\pm$ 0.07  & 238.21\%  \\
Alpha &  9.68 $\pm$ 3.17  &  40.66 $\pm$ 0.25  & 320.04\%  \\
UCI &  15.98 $\pm$ 4.06  &  35.59 $\pm$ 1.58  & 122.72\%  \\
Title &  17.87 $\pm$ 2.40  &  50.84 $\pm$ 0.05  & 184.50\%  \\
Body &  12.33 $\pm$ 2.22  &  40.97 $\pm$ 0.47  & 232.28\%  \\
AS733 &  22.62 $\pm$ 1.04  &  45.95 $\pm$ 0.79  & 103.14\%  \\
SBM &  3.78 $\pm$ 1.87  &  28.96 $\pm$ 0.70  & 666.14\%  \\
SO &  22.39 $\pm$ 2.48  &  48.83 $\pm$ 0.14  & 118.09\%  \\
    \bottomrule
  \end{tabular}
\end{table}

As presented in Table~\ref{tab:ablation}, the findings reveal that the Hawkes processes-based GAT model significantly enhances the overall performance, surpassing the standard GAT by a considerable margin across all datasets. Particularly, on the SBM dataset, the Mrr@100 metric exhibited a notable increase, skyrocketing from 3.78\% to 28.96\%. This notable improvement underscores the efficacy and versatility of our proposed methodology.

\subsection{Hyper-parameter Sensitivity Analysis}

\begin{figure}[ht]
  \centering
  \includegraphics[width=\linewidth]{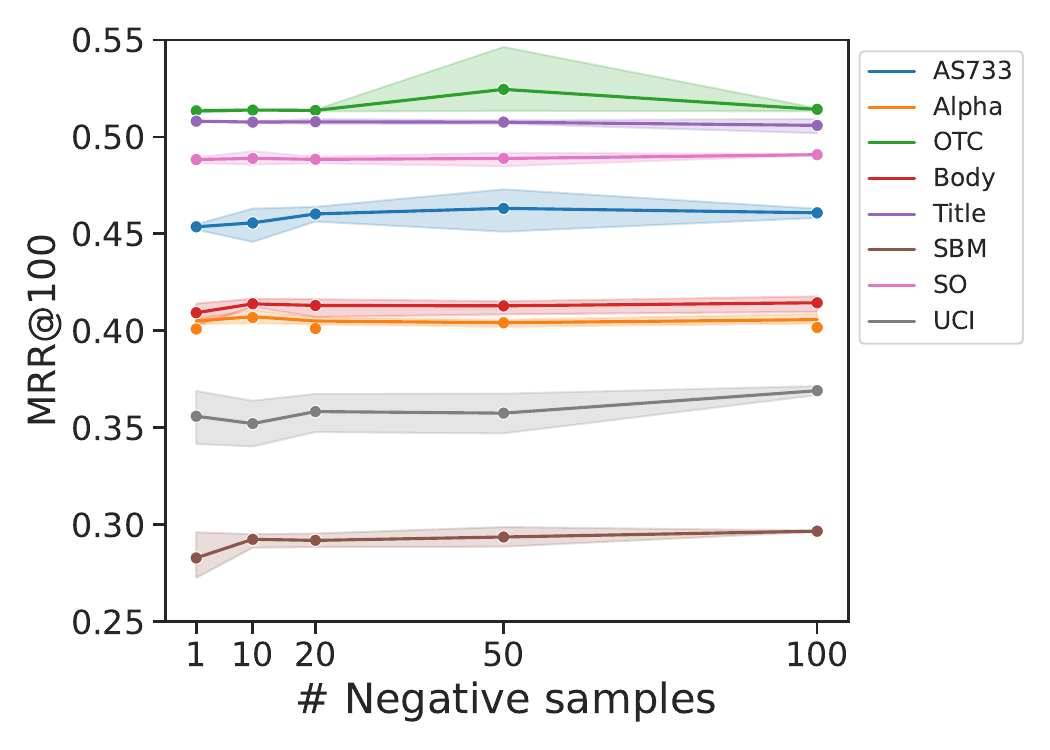}
  \caption{Performance of Hawkes-GAT with varying numbers of negative samples during training. }
  \Description{}
  \label{fig:negative_all}
\end{figure}

\subsubsection{Negative sampling} helps train the model to distinguish between positive and negative pairs, which is widely used in link prediction tasks. Let $k$ denote the number of negative samples for training. Previous works tend to have $k$ larger than one. While $k$ is essential for the efficiency of the mini-batch algorithm, we study the model performance with varying numbers of negative samples. 
As shown in Figure~\ref{fig:negative_all}, the average model performance across the eight datasets does not exhibit significant changes as $k$ increases. Take the UCI dataset as an example (Figure~\ref{fig:negative_uci}), an analysis using ANOVA \cite{girden1992anova} on the five experiment groups resulted in a p-value of $0.45969$, which fails to reject the null hypothesis, suggesting no substantial variance in the means across these experimental sets. Consequently, we have the conclusion that the average performance of the proposed method is insensitive to $k$ for our proposed method.
\begin{figure}[ht]
  \centering
  \includegraphics[width=\linewidth]{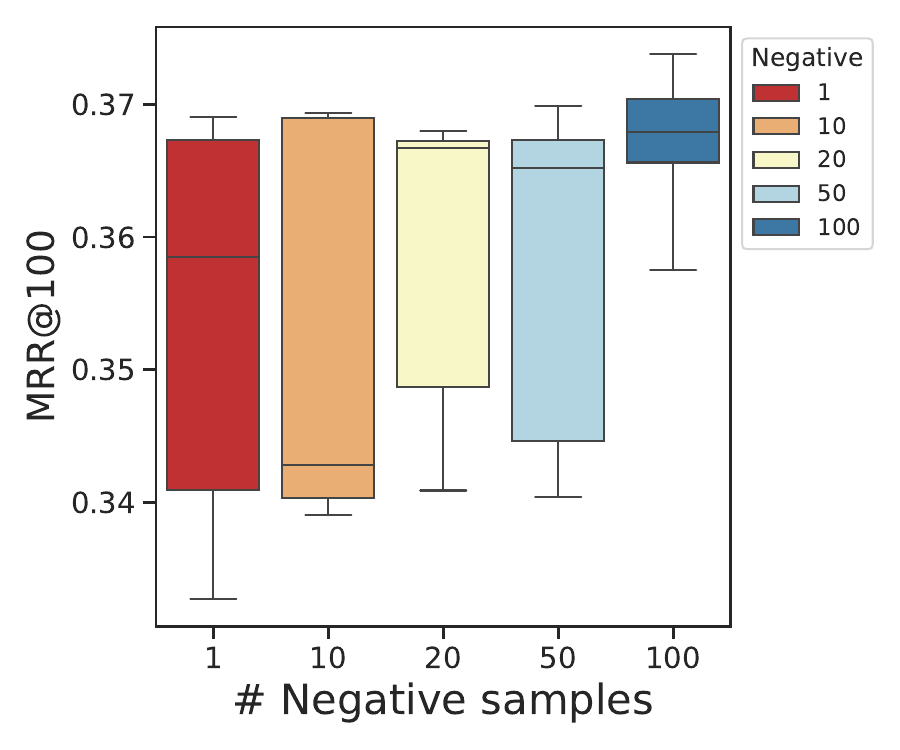}
  \caption{The distribution of MRR@100 on the UCI dataset with different negative samples for training.}
  \Description{}
  \label{fig:negative_uci}
\end{figure}

\begin{figure}
  \centering
  \includegraphics[width=\linewidth]{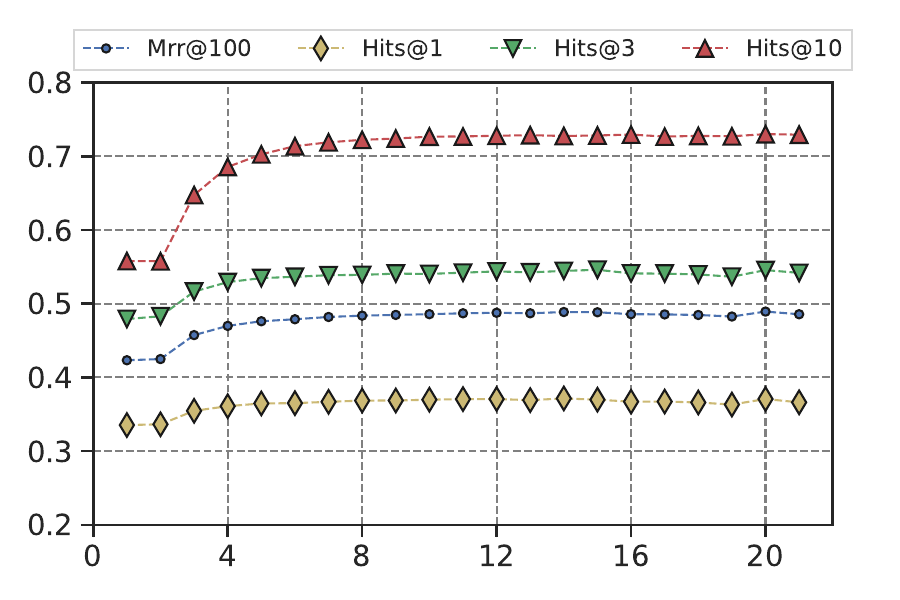}
  \caption{Performance of Hawkes-GAT on Reddit-Title with various sliding window sizes.}
  \Description{}
  \label{fig:hyper}
\end{figure}

\subsubsection{Sliding window size} serves as another key hyperparameter. To investigate the model's sensitivity to the window size, various window sizes were used in training the Reddit-Title dataset to analyze the performance variations, as illustrated in Figure~\ref{fig:hyper}. The results demonstrate that a window size of 4 represents a critical point in the model's performance, indicating the presence of long-distance temporal dependencies within the data. Insufficient input snapshots lead to the omission of essential information. Furthermore, depicted in the figure, beyond a certain window size, the impact of the number of input windows on future predictions notably diminishes. This observation aligns with the Hawkes process assumption used in this study, suggesting that the influence of past events on future predictions gradually wanes over time. Taken together, these findings underscore the robustness of our proposed Hawkes-GNNs.

\section{Related Works}

\subsection{Dynamic Graph Neural Networks}
Dynamic graph representation learning aims to learn the temporal low-dimensional representations of nodes as the graph evolves over time, mainly categorized into continuous-time dynamic graph (CTDG) and discrete-time dynamic graph (DTDG) methods based on the form of dynamic graphs. CTDG based methods treat dynamic graphs as event streams with accurate timestamps, generating dynamic node embeddings by iteratively processing information gathered from temporal neighbors \cite{trivedi2019dyrep, kumar2019predicting}. Despite the success of continuous-time methods \cite{tgat_iclr20, tgn_icml_grl2020, wang2021apan}, in practical applications, many datasets lack precise timestamp information due to historical reasons, making continuous-time methods inapplicable. 
In addition, while our work adopts the concept of temporal graphs from CTDG, it differs in its approach under the DTDG framework by modeling and predicting the entire graph as a whole. This distinction sets our research apart from prior studies.

DTDG represents the dynamic graph through a sequence of snapshots. It leverages multiple previous snapshots to predict the events in subsequent snapshots. The intuitive idea is to combine GNN with a time encoder such as RNN \cite{hochreiter1997long, chung2014empirical} or Transformer \cite{vaswani2017attention}, which, however, has high time complexity and space complexity. Recent studies have highlighted that the incorporation of extra time encoders can lead to overfitting \cite{Zhu23}. Alternatively, a single snapshot as input has been explored by employing latent random variables \cite{hajiramezanali2019variational}, hyperbolic space \cite{yang2021discrete} or meta-learning \cite{You22}. These approaches enable the system to process a single snapshot as input, offering a promising solution for scalability. 
By sequentially training the snapshots, existing static graph scaling techniques can be utilized. However, it should be noted that the training time scales linearly with the total number of snapshots. Moreover, disregarding long-distance time dependencies may lead to overfitting issues. What sets our research apart from existing methodologies is that we fuse the multiple input snapshots, thereby obviating the necessity for additional temporal encoders.

\subsection{Scalable Graph Neural Networks}
GNNs are typically executed in a full-batch manner, which makes it challenging to scale to large graphs in practice for limited GPU memory and training time. To facilitate training on large-scale datasets, mini-batch training methods were proposed for static GNNs such as graph sampling \cite{hamilton2017inductive, chen2018fastgcn, chiang2019cluster}, graph sparsification \cite{rong2019dropedge, li2020sgcn}, and graph partitioning \cite{chiang2019cluster, md2021distgnn} . 
The basic idea is training the network on only a subset of the entire graph data at a time, rather than using the entire graph. The same issue also arises in DTDGs with the additional challenge that there are multiple snapshots. Few studies have explored how to align selected nodes across snapshots for mini-batch training on DTDGs. In contrast, our proposed Input Snapshots Fusion facilitates mini-batch training on both multiple snapshots and temporal graph approaches, enhancing the scalability of large datasets.

\subsection{Hawkes Processes Based Graph Learning}
Hawkes processes are powerful temporal point processes widely used in modeling event sequences \cite{hawkes1971spectra}, and have been extensively studied to adapt to different scenarios. Prior studies \cite{Shang_Sun_2019, zuo2018embedding, 10.1145/3357384.3357943, 10.1145/3485447.3512164, 10.1007/978-3-030-86486-6_24} often adhere to the formulation of the Equation~\ref{equ:hawkes} to predict the probability of future links, where the neighbor influence component duplicates message aggregation mechanism in GNNs. In contrast, we propose a deeper integration of Hawkes processes with GNNs through time-decayed message passing which eliminates the need for additional parameters from Hawkes processes, 
setting our approach apart from previous works.

\section{Conclusion}
This study proposes a novel approach to enhance the scalability of discrete-time dynamic graph models by combining multiple snapshots within the input sliding window into a single temporal graph. This method effectively decouples computational complexity from the number of snapshots, enabling the use of mini-batch training methods. 
To model the generated large temporal graph, we employ Hawkes excitation matrix to represent the temporal edges, which provides modeling of the temporal graph as denoising with time decay smoothing assumption. 
Building on this, we propose Hawkes processes-based GNNs, which capture graph dynamics effectively while being more resource-efficient than previous methods.
Extensive experiments demonstrate the scalability, robustness, and versatility of our framework.  
This research focuses solely on the link prediction task in discrete-time dynamic graphs due to data constraints. Future directions involve exploring scalable methods such as graph partition and expanding our framework to diverse dynamic graph representation learning tasks including node classification and link classification.

\begin{acks}
Q. Qi and H. Chen were supported in part by the National Key Research and Development Program of China(2023YFB4502400) and the National Natural Science Foundation of China under Grant 62271452.
\end{acks}

\bibliographystyle{ACM-Reference-Format}
\balance
\bibliography{kdd}


\begin{thebibliography}{50}


\ifx \showCODEN    \undefined \def \showCODEN     #1{\unskip}     \fi
\ifx \showDOI      \undefined \def \showDOI       #1{#1}\fi
\ifx \showISBNx    \undefined \def \showISBNx     #1{\unskip}     \fi
\ifx \showISBNxiii \undefined \def \showISBNxiii  #1{\unskip}     \fi
\ifx \showISSN     \undefined \def \showISSN      #1{\unskip}     \fi
\ifx \showLCCN     \undefined \def \showLCCN      #1{\unskip}     \fi
\ifx \shownote     \undefined \def \shownote      #1{#1}          \fi
\ifx \showarticletitle \undefined \def \showarticletitle #1{#1}   \fi
\ifx \showURL      \undefined \def \showURL       {\relax}        \fi
\providecommand\bibfield[2]{#2}
\providecommand\bibinfo[2]{#2}
\providecommand\natexlab[1]{#1}
\providecommand\showeprint[2][]{arXiv:#2}

\bibitem[Chakaravarthy et~al\mbox{.}(2021)]%
        {10.1145/3458817.3480858}
\bibfield{author}{\bibinfo{person}{Venkatesan~T. Chakaravarthy}, \bibinfo{person}{Shivmaran~S. Pandian}, \bibinfo{person}{Saurabh Raje}, \bibinfo{person}{Yogish Sabharwal}, \bibinfo{person}{Toyotaro Suzumura}, {and} \bibinfo{person}{Shashanka Ubaru}.} \bibinfo{year}{2021}\natexlab{}.
\newblock \showarticletitle{Efficient scaling of dynamic graph neural networks}. In \bibinfo{booktitle}{\emph{Proceedings of the International Conference for High Performance Computing, Networking, Storage and Analysis}} (St. Louis, Missouri) \emph{(\bibinfo{series}{SC '21})}. \bibinfo{publisher}{Association for Computing Machinery}, \bibinfo{address}{New York, NY, USA}, Article \bibinfo{articleno}{77}, \bibinfo{numpages}{15}~pages.
\newblock
\showISBNx{9781450384421}
\urldef\tempurl%
\url{https://doi.org/10.1145/3458817.3480858}
\showDOI{\tempurl}


\bibitem[Chen et~al\mbox{.}(2018)]%
        {chen2018fastgcn}
\bibfield{author}{\bibinfo{person}{Jie Chen}, \bibinfo{person}{Tengfei Ma}, {and} \bibinfo{person}{Cao Xiao}.} \bibinfo{year}{2018}\natexlab{}.
\newblock \showarticletitle{FastGCN: Fast learning with graph convolu-tional networks via importance sampling}. In \bibinfo{booktitle}{\emph{International Conference on Learning Representations}}. International Conference on Learning Representations, ICLR.
\newblock


\bibitem[Chiang et~al\mbox{.}(2019)]%
        {chiang2019cluster}
\bibfield{author}{\bibinfo{person}{Wei-Lin Chiang}, \bibinfo{person}{Xuanqing Liu}, \bibinfo{person}{Si Si}, \bibinfo{person}{Yang Li}, \bibinfo{person}{Samy Bengio}, {and} \bibinfo{person}{Cho-Jui Hsieh}.} \bibinfo{year}{2019}\natexlab{}.
\newblock \showarticletitle{Cluster-gcn: An efficient algorithm for training deep and large graph convolutional networks}. In \bibinfo{booktitle}{\emph{Proceedings of the 25th ACM SIGKDD international conference on knowledge discovery \& data mining}}. \bibinfo{pages}{257--266}.
\newblock


\bibitem[Chung et~al\mbox{.}(2014)]%
        {chung2014empirical}
\bibfield{author}{\bibinfo{person}{Junyoung Chung}, \bibinfo{person}{Caglar Gulcehre}, \bibinfo{person}{Kyunghyun Cho}, {and} \bibinfo{person}{Yoshua Bengio}.} \bibinfo{year}{2014}\natexlab{}.
\newblock \showarticletitle{Empirical evaluation of gated recurrent neural networks on sequence modeling}. In \bibinfo{booktitle}{\emph{NIPS 2014 Workshop on Deep Learning, December 2014}}.
\newblock


\bibitem[Cong et~al\mbox{.}(2023a)]%
        {congwe2023}
\bibfield{author}{\bibinfo{person}{Weilin Cong}, \bibinfo{person}{Si Zhang}, \bibinfo{person}{Jian Kang}, \bibinfo{person}{Baichuan Yuan}, \bibinfo{person}{Hao Wu}, \bibinfo{person}{Xin Zhou}, \bibinfo{person}{Hanghang Tong}, {and} \bibinfo{person}{Mehrdad Mahdavi}.} \bibinfo{year}{2023}\natexlab{a}.
\newblock \showarticletitle{Do We Really Need Complicated Model Architectures For Temporal Networks?}. In \bibinfo{booktitle}{\emph{The Eleventh International Conference on Learning Representations}}.
\newblock


\bibitem[Cong et~al\mbox{.}(2023b)]%
        {congwe}
\bibfield{author}{\bibinfo{person}{Weilin Cong}, \bibinfo{person}{Si Zhang}, \bibinfo{person}{Jian Kang}, \bibinfo{person}{Baichuan Yuan}, \bibinfo{person}{Hao Wu}, \bibinfo{person}{Xin Zhou}, \bibinfo{person}{Hanghang Tong}, {and} \bibinfo{person}{Mehrdad Mahdavi}.} \bibinfo{year}{2023}\natexlab{b}.
\newblock \showarticletitle{Do We Really Need Complicated Model Architectures For Temporal Networks?}. In \bibinfo{booktitle}{\emph{The Eleventh International Conference on Learning Representations}}.
\newblock


\bibitem[da~Xu et~al\mbox{.}(2020)]%
        {tgat_iclr20}
\bibfield{author}{\bibinfo{person}{da Xu}, \bibinfo{person}{chuanwei ruan}, \bibinfo{person}{evren korpeoglu}, \bibinfo{person}{sushant kumar}, {and} \bibinfo{person}{kannan achan}.} \bibinfo{year}{2020}\natexlab{}.
\newblock \showarticletitle{Inductive representation learning on temporal graphs}. In \bibinfo{booktitle}{\emph{International Conference on Learning Representations (ICLR)}}.
\newblock


\bibitem[Fan et~al\mbox{.}(2019)]%
        {Fan_Ma_Li_He_Zhao_Tang_Yin_2019}
\bibfield{author}{\bibinfo{person}{Wenqi Fan}, \bibinfo{person}{Yao Ma}, \bibinfo{person}{Qing Li}, \bibinfo{person}{Yuan He}, \bibinfo{person}{Eric Zhao}, \bibinfo{person}{Jiliang Tang}, {and} \bibinfo{person}{Dawei Yin}.} \bibinfo{year}{2019}\natexlab{}.
\newblock \showarticletitle{Graph Neural Networks for Social Recommendation}. In \bibinfo{booktitle}{\emph{The World Wide Web Conference}}.
\newblock
\urldef\tempurl%
\url{https://doi.org/10.1145/3308558.3313488}
\showDOI{\tempurl}


\bibitem[Fey and Lenssen(2019)]%
        {Fey/Lenssen/2019}
\bibfield{author}{\bibinfo{person}{Matthias Fey} {and} \bibinfo{person}{Jan~E. Lenssen}.} \bibinfo{year}{2019}\natexlab{}.
\newblock \showarticletitle{Fast Graph Representation Learning with {PyTorch Geometric}}. In \bibinfo{booktitle}{\emph{ICLR Workshop on Representation Learning on Graphs and Manifolds}}.
\newblock


\bibitem[Gilmer et~al\mbox{.}(2017)]%
        {gilmer2017neural}
\bibfield{author}{\bibinfo{person}{Justin Gilmer}, \bibinfo{person}{Samuel~S Schoenholz}, \bibinfo{person}{Patrick~F Riley}, \bibinfo{person}{Oriol Vinyals}, {and} \bibinfo{person}{George~E Dahl}.} \bibinfo{year}{2017}\natexlab{}.
\newblock \showarticletitle{Neural message passing for quantum chemistry}. In \bibinfo{booktitle}{\emph{International conference on machine learning}}. PMLR, \bibinfo{pages}{1263--1272}.
\newblock


\bibitem[Girden(1992)]%
        {girden1992anova}
\bibfield{author}{\bibinfo{person}{Ellen~R Girden}.} \bibinfo{year}{1992}\natexlab{}.
\newblock \bibinfo{booktitle}{\emph{ANOVA: Repeated measures}}.
\newblock Number~84. \bibinfo{publisher}{Sage}.
\newblock


\bibitem[Hajiramezanali et~al\mbox{.}(2019)]%
        {hajiramezanali2019variational}
\bibfield{author}{\bibinfo{person}{Ehsan Hajiramezanali}, \bibinfo{person}{Arman Hasanzadeh}, \bibinfo{person}{Krishna Narayanan}, \bibinfo{person}{Nick Duffield}, \bibinfo{person}{Mingyuan Zhou}, {and} \bibinfo{person}{Xiaoning Qian}.} \bibinfo{year}{2019}\natexlab{}.
\newblock \showarticletitle{Variational graph recurrent neural networks}.
\newblock \bibinfo{journal}{\emph{Advances in neural information processing systems}}  \bibinfo{volume}{32} (\bibinfo{year}{2019}).
\newblock


\bibitem[Hamilton et~al\mbox{.}(2017)]%
        {hamilton2017inductive}
\bibfield{author}{\bibinfo{person}{Will Hamilton}, \bibinfo{person}{Zhitao Ying}, {and} \bibinfo{person}{Jure Leskovec}.} \bibinfo{year}{2017}\natexlab{}.
\newblock \showarticletitle{Inductive representation learning on large graphs}.
\newblock \bibinfo{journal}{\emph{Advances in neural information processing systems}}  \bibinfo{volume}{30} (\bibinfo{year}{2017}).
\newblock


\bibitem[Hawkes(1971a)]%
        {Hawkes1971SpectraOS}
\bibfield{author}{\bibinfo{person}{Alan~G. Hawkes}.} \bibinfo{year}{1971}\natexlab{a}.
\newblock \showarticletitle{Spectra of some self-exciting and mutually exciting point processes}.
\newblock \bibinfo{journal}{\emph{Biometrika}}  \bibinfo{volume}{58} (\bibinfo{year}{1971}), \bibinfo{pages}{83--90}.
\newblock
\urldef\tempurl%
\url{https://api.semanticscholar.org/CorpusID:14122089}
\showURL{%
\tempurl}


\bibitem[Hawkes(1971b)]%
        {hawkes1971spectra}
\bibfield{author}{\bibinfo{person}{Alan~G Hawkes}.} \bibinfo{year}{1971}\natexlab{b}.
\newblock \showarticletitle{Spectra of some self-exciting and mutually exciting point processes}.
\newblock \bibinfo{journal}{\emph{Biometrika}} \bibinfo{volume}{58}, \bibinfo{number}{1} (\bibinfo{year}{1971}), \bibinfo{pages}{83--90}.
\newblock


\bibitem[Hochreiter and Schmidhuber(1997)]%
        {hochreiter1997long}
\bibfield{author}{\bibinfo{person}{Sepp Hochreiter} {and} \bibinfo{person}{J{\"u}rgen Schmidhuber}.} \bibinfo{year}{1997}\natexlab{}.
\newblock \showarticletitle{Long short-term memory}.
\newblock \bibinfo{journal}{\emph{Neural computation}} \bibinfo{volume}{9}, \bibinfo{number}{8} (\bibinfo{year}{1997}), \bibinfo{pages}{1735--1780}.
\newblock


\bibitem[Hu et~al\mbox{.}(2020)]%
        {hu2020open}
\bibfield{author}{\bibinfo{person}{Weihua Hu}, \bibinfo{person}{Matthias Fey}, \bibinfo{person}{Marinka Zitnik}, \bibinfo{person}{Yuxiao Dong}, \bibinfo{person}{Hongyu Ren}, \bibinfo{person}{Bowen Liu}, \bibinfo{person}{Michele Catasta}, {and} \bibinfo{person}{Jure Leskovec}.} \bibinfo{year}{2020}\natexlab{}.
\newblock \showarticletitle{Open graph benchmark: Datasets for machine learning on graphs}.
\newblock \bibinfo{journal}{\emph{Advances in neural information processing systems}}  \bibinfo{volume}{33} (\bibinfo{year}{2020}), \bibinfo{pages}{22118--22133}.
\newblock


\bibitem[Ioffe and Szegedy(2015)]%
        {ioffe2015batch}
\bibfield{author}{\bibinfo{person}{Sergey Ioffe} {and} \bibinfo{person}{Christian Szegedy}.} \bibinfo{year}{2015}\natexlab{}.
\newblock \showarticletitle{Batch normalization: Accelerating deep network training by reducing internal covariate shift}. In \bibinfo{booktitle}{\emph{International conference on machine learning}}. pmlr, \bibinfo{pages}{448--456}.
\newblock


\bibitem[Ji et~al\mbox{.}(2021)]%
        {10.1007/978-3-030-86486-6_24}
\bibfield{author}{\bibinfo{person}{Yugang Ji}, \bibinfo{person}{Tianrui Jia}, \bibinfo{person}{Yuan Fang}, {and} \bibinfo{person}{Chuan Shi}.} \bibinfo{year}{2021}\natexlab{}.
\newblock \showarticletitle{Dynamic Heterogeneous Graph Embedding via Heterogeneous Hawkes Process}. In \bibinfo{booktitle}{\emph{Machine Learning and Knowledge Discovery in Databases. Research Track: European Conference, ECML PKDD 2021, Bilbao, Spain, September 13–17, 2021, Proceedings, Part I}} (Bilbao, Spain). \bibinfo{publisher}{Springer-Verlag}, \bibinfo{address}{Berlin, Heidelberg}, \bibinfo{pages}{388–403}.
\newblock
\showISBNx{978-3-030-86485-9}
\urldef\tempurl%
\url{https://doi.org/10.1007/978-3-030-86486-6_24}
\showDOI{\tempurl}


\bibitem[Kazemi et~al\mbox{.}(2020)]%
        {kazemi2020representation}
\bibfield{author}{\bibinfo{person}{Seyed~Mehran Kazemi}, \bibinfo{person}{Rishab Goel}, \bibinfo{person}{Kshitij Jain}, \bibinfo{person}{Ivan Kobyzev}, \bibinfo{person}{Akshay Sethi}, \bibinfo{person}{Peter Forsyth}, {and} \bibinfo{person}{Pascal Poupart}.} \bibinfo{year}{2020}\natexlab{}.
\newblock \showarticletitle{Representation learning for dynamic graphs: A survey}.
\newblock \bibinfo{journal}{\emph{The Journal of Machine Learning Research}} \bibinfo{volume}{21}, \bibinfo{number}{1} (\bibinfo{year}{2020}), \bibinfo{pages}{2648--2720}.
\newblock


\bibitem[Kingma and Ba(2014)]%
        {kingma2014adam}
\bibfield{author}{\bibinfo{person}{Diederik~P Kingma} {and} \bibinfo{person}{Jimmy Ba}.} \bibinfo{year}{2014}\natexlab{}.
\newblock \showarticletitle{Adam: A method for stochastic optimization}.
\newblock \bibinfo{journal}{\emph{arXiv preprint arXiv:1412.6980}} (\bibinfo{year}{2014}).
\newblock


\bibitem[Kipf and Welling(2017)]%
        {kipf2017semi}
\bibfield{author}{\bibinfo{person}{Thomas~N. Kipf} {and} \bibinfo{person}{Max Welling}.} \bibinfo{year}{2017}\natexlab{}.
\newblock \showarticletitle{Semi-Supervised Classification with Graph Convolutional Networks}. In \bibinfo{booktitle}{\emph{International Conference on Learning Representations (ICLR)}}.
\newblock


\bibitem[Kumar et~al\mbox{.}(2019)]%
        {kumar2019predicting}
\bibfield{author}{\bibinfo{person}{Srijan Kumar}, \bibinfo{person}{Xikun Zhang}, {and} \bibinfo{person}{Jure Leskovec}.} \bibinfo{year}{2019}\natexlab{}.
\newblock \showarticletitle{Predicting dynamic embedding trajectory in temporal interaction networks}. In \bibinfo{booktitle}{\emph{Proceedings of the 25th ACM SIGKDD international conference on knowledge discovery \& data mining}}. \bibinfo{pages}{1269--1278}.
\newblock


\bibitem[Li et~al\mbox{.}(2020)]%
        {li2020sgcn}
\bibfield{author}{\bibinfo{person}{Jiayu Li}, \bibinfo{person}{Tianyun Zhang}, \bibinfo{person}{Hao Tian}, \bibinfo{person}{Shengmin Jin}, \bibinfo{person}{Makan Fardad}, {and} \bibinfo{person}{Reza Zafarani}.} \bibinfo{year}{2020}\natexlab{}.
\newblock \showarticletitle{Sgcn: A graph sparsifier based on graph convolutional networks}. In \bibinfo{booktitle}{\emph{Advances in Knowledge Discovery and Data Mining: 24th Pacific-Asia Conference, PAKDD 2020, Singapore, May 11--14, 2020, Proceedings, Part I 24}}. Springer, \bibinfo{pages}{275--287}.
\newblock


\bibitem[Li and Zhu(2021)]%
        {li2021spatial}
\bibfield{author}{\bibinfo{person}{Mengzhang Li} {and} \bibinfo{person}{Zhanxing Zhu}.} \bibinfo{year}{2021}\natexlab{}.
\newblock \showarticletitle{Spatial-temporal fusion graph neural networks for traffic flow forecasting}. In \bibinfo{booktitle}{\emph{Proceedings of the AAAI conference on artificial intelligence}}, Vol.~\bibinfo{volume}{35}. \bibinfo{pages}{4189--4196}.
\newblock


\bibitem[Loshchilov and Hutter(2016)]%
        {loshchilov2016sgdr}
\bibfield{author}{\bibinfo{person}{Ilya Loshchilov} {and} \bibinfo{person}{Frank Hutter}.} \bibinfo{year}{2016}\natexlab{}.
\newblock \showarticletitle{Sgdr: Stochastic gradient descent with warm restarts}.
\newblock \bibinfo{journal}{\emph{arXiv preprint arXiv:1608.03983}} (\bibinfo{year}{2016}).
\newblock


\bibitem[Lu et~al\mbox{.}(2019)]%
        {10.1145/3357384.3357943}
\bibfield{author}{\bibinfo{person}{Yuanfu Lu}, \bibinfo{person}{Xiao Wang}, \bibinfo{person}{Chuan Shi}, \bibinfo{person}{Philip~S. Yu}, {and} \bibinfo{person}{Yanfang Ye}.} \bibinfo{year}{2019}\natexlab{}.
\newblock \showarticletitle{Temporal Network Embedding with Micro- and Macro-dynamics}. In \bibinfo{booktitle}{\emph{Proceedings of the 28th ACM International Conference on Information and Knowledge Management}} (Beijing, China) \emph{(\bibinfo{series}{CIKM '19})}. \bibinfo{publisher}{Association for Computing Machinery}, \bibinfo{address}{New York, NY, USA}, \bibinfo{pages}{469–478}.
\newblock
\showISBNx{9781450369763}
\urldef\tempurl%
\url{https://doi.org/10.1145/3357384.3357943}
\showDOI{\tempurl}


\bibitem[Ma et~al\mbox{.}(2021)]%
        {ma2021unified}
\bibfield{author}{\bibinfo{person}{Yao Ma}, \bibinfo{person}{Xiaorui Liu}, \bibinfo{person}{Tong Zhao}, \bibinfo{person}{Yozen Liu}, \bibinfo{person}{Jiliang Tang}, {and} \bibinfo{person}{Neil Shah}.} \bibinfo{year}{2021}\natexlab{}.
\newblock \showarticletitle{A unified view on graph neural networks as graph signal denoising}. In \bibinfo{booktitle}{\emph{Proceedings of the 30th ACM International Conference on Information \& Knowledge Management}}. \bibinfo{pages}{1202--1211}.
\newblock


\bibitem[Md et~al\mbox{.}(2021)]%
        {md2021distgnn}
\bibfield{author}{\bibinfo{person}{Vasimuddin Md}, \bibinfo{person}{Sanchit Misra}, \bibinfo{person}{Guixiang Ma}, \bibinfo{person}{Ramanarayan Mohanty}, \bibinfo{person}{Evangelos Georganas}, \bibinfo{person}{Alexander Heinecke}, \bibinfo{person}{Dhiraj Kalamkar}, \bibinfo{person}{Nesreen~K Ahmed}, {and} \bibinfo{person}{Sasikanth Avancha}.} \bibinfo{year}{2021}\natexlab{}.
\newblock \showarticletitle{Distgnn: Scalable distributed training for large-scale graph neural networks}. In \bibinfo{booktitle}{\emph{Proceedings of the International Conference for High Performance Computing, Networking, Storage and Analysis}}. \bibinfo{pages}{1--14}.
\newblock


\bibitem[Pareja et~al\mbox{.}(2020)]%
        {egcn20}
\bibfield{author}{\bibinfo{person}{Aldo Pareja}, \bibinfo{person}{Giacomo Domeniconi}, \bibinfo{person}{Jie Chen}, \bibinfo{person}{Tengfei Ma}, \bibinfo{person}{Toyotaro Suzumura}, \bibinfo{person}{Hiroki Kanezashi}, \bibinfo{person}{Tim Kaler}, \bibinfo{person}{Tao~B. Schardl}, {and} \bibinfo{person}{Charles~E. Leiserson}.} \bibinfo{year}{2020}\natexlab{}.
\newblock \showarticletitle{{EvolveGCN}: Evolving Graph Convolutional Networks for Dynamic Graphs}. In \bibinfo{booktitle}{\emph{Proceedings of the Thirty-Fourth AAAI Conference on Artificial Intelligence}}.
\newblock


\bibitem[Reiser et~al\mbox{.}(2022)]%
        {reiser2022graph}
\bibfield{author}{\bibinfo{person}{Patrick Reiser}, \bibinfo{person}{Marlen Neubert}, \bibinfo{person}{Andr{\'e} Eberhard}, \bibinfo{person}{Luca Torresi}, \bibinfo{person}{Chen Zhou}, \bibinfo{person}{Chen Shao}, \bibinfo{person}{Houssam Metni}, \bibinfo{person}{Clint van Hoesel}, \bibinfo{person}{Henrik Schopmans}, \bibinfo{person}{Timo Sommer}, {et~al\mbox{.}}} \bibinfo{year}{2022}\natexlab{}.
\newblock \showarticletitle{Graph neural networks for materials science and chemistry}.
\newblock \bibinfo{journal}{\emph{Communications Materials}} \bibinfo{volume}{3}, \bibinfo{number}{1} (\bibinfo{year}{2022}), \bibinfo{pages}{93}.
\newblock


\bibitem[Rong et~al\mbox{.}(2019)]%
        {rong2019dropedge}
\bibfield{author}{\bibinfo{person}{Yu Rong}, \bibinfo{person}{Wenbing Huang}, \bibinfo{person}{Tingyang Xu}, {and} \bibinfo{person}{Junzhou Huang}.} \bibinfo{year}{2019}\natexlab{}.
\newblock \showarticletitle{DropEdge: Towards Deep Graph Convolutional Networks on Node Classification}. In \bibinfo{booktitle}{\emph{International Conference on Learning Representations}}.
\newblock


\bibitem[Rossi et~al\mbox{.}(2020)]%
        {tgn_icml_grl2020}
\bibfield{author}{\bibinfo{person}{Emanuele Rossi}, \bibinfo{person}{Ben Chamberlain}, \bibinfo{person}{Fabrizio Frasca}, \bibinfo{person}{Davide Eynard}, \bibinfo{person}{Federico Monti}, {and} \bibinfo{person}{Michael Bronstein}.} \bibinfo{year}{2020}\natexlab{}.
\newblock \showarticletitle{Temporal Graph Networks for Deep Learning on Dynamic Graphs}. In \bibinfo{booktitle}{\emph{ICML 2020 Workshop on Graph Representation Learning}}.
\newblock


\bibitem[Sankar et~al\mbox{.}(2020)]%
        {sankar2020dysat}
\bibfield{author}{\bibinfo{person}{Aravind Sankar}, \bibinfo{person}{Yanhong Wu}, \bibinfo{person}{Liang Gou}, \bibinfo{person}{Wei Zhang}, {and} \bibinfo{person}{Hao Yang}.} \bibinfo{year}{2020}\natexlab{}.
\newblock \showarticletitle{Dysat: Deep neural representation learning on dynamic graphs via self-attention networks}. In \bibinfo{booktitle}{\emph{Proceedings of the 13th international conference on web search and data mining}}. \bibinfo{pages}{519--527}.
\newblock


\bibitem[Shang and Sun(2019)]%
        {Shang_Sun_2019}
\bibfield{author}{\bibinfo{person}{Jin Shang} {and} \bibinfo{person}{Mingxuan Sun}.} \bibinfo{year}{2019}\natexlab{}.
\newblock \showarticletitle{Geometric Hawkes Processes with Graph Convolutional Recurrent Neural Networks}.
\newblock \bibinfo{journal}{\emph{Proceedings of the AAAI Conference on Artificial Intelligence}} \bibinfo{volume}{33}, \bibinfo{number}{01} (\bibinfo{date}{Jul.} \bibinfo{year}{2019}), \bibinfo{pages}{4878--4885}.
\newblock
\urldef\tempurl%
\url{https://doi.org/10.1609/aaai.v33i01.33014878}
\showDOI{\tempurl}


\bibitem[Trivedi et~al\mbox{.}(2019)]%
        {trivedi2019dyrep}
\bibfield{author}{\bibinfo{person}{Rakshit Trivedi}, \bibinfo{person}{Mehrdad Farajtabar}, \bibinfo{person}{Prasenjeet Biswal}, {and} \bibinfo{person}{Hongyuan Zha}.} \bibinfo{year}{2019}\natexlab{}.
\newblock \showarticletitle{Dyrep: Learning representations over dynamic graphs}. In \bibinfo{booktitle}{\emph{International conference on learning representations}}.
\newblock


\bibitem[Vaswani et~al\mbox{.}(2017)]%
        {vaswani2017attention}
\bibfield{author}{\bibinfo{person}{Ashish Vaswani}, \bibinfo{person}{Noam Shazeer}, \bibinfo{person}{Niki Parmar}, \bibinfo{person}{Jakob Uszkoreit}, \bibinfo{person}{Llion Jones}, \bibinfo{person}{Aidan~N Gomez}, \bibinfo{person}{{\L}ukasz Kaiser}, {and} \bibinfo{person}{Illia Polosukhin}.} \bibinfo{year}{2017}\natexlab{}.
\newblock \showarticletitle{Attention is all you need}.
\newblock \bibinfo{journal}{\emph{Advances in neural information processing systems}}  \bibinfo{volume}{30} (\bibinfo{year}{2017}).
\newblock


\bibitem[Veličković et~al\mbox{.}(2017)]%
        {velikovi2017graph}
\bibfield{author}{\bibinfo{person}{Petar Veličković}, \bibinfo{person}{Guillem Cucurull}, \bibinfo{person}{Arantxa Casanova}, \bibinfo{person}{Adriana Romero}, \bibinfo{person}{Pietro Liò}, {and} \bibinfo{person}{Yoshua Bengio}.} \bibinfo{year}{2017}\natexlab{}.
\newblock \showarticletitle{Graph Attention Networks}.
\newblock \bibinfo{journal}{\emph{6th International Conference on Learning Representations}} (\bibinfo{year}{2017}).
\newblock


\bibitem[Wang et~al\mbox{.}(2021)]%
        {wang2021apan}
\bibfield{author}{\bibinfo{person}{Xuhong Wang}, \bibinfo{person}{Ding Lyu}, \bibinfo{person}{Mengjian Li}, \bibinfo{person}{Yang Xia}, \bibinfo{person}{Qi Yang}, \bibinfo{person}{Xinwen Wang}, \bibinfo{person}{Xinguang Wang}, \bibinfo{person}{Ping Cui}, \bibinfo{person}{Yupu Yang}, \bibinfo{person}{Bowen Sun}, {et~al\mbox{.}}} \bibinfo{year}{2021}\natexlab{}.
\newblock \showarticletitle{Apan: Asynchronous propagation attention network for real-time temporal graph embedding}. In \bibinfo{booktitle}{\emph{Proceedings of the 2021 international conference on management of data}}. \bibinfo{pages}{2628--2638}.
\newblock


\bibitem[Wen and Fang(2022)]%
        {10.1145/3485447.3512164}
\bibfield{author}{\bibinfo{person}{Zhihao Wen} {and} \bibinfo{person}{Yuan Fang}.} \bibinfo{year}{2022}\natexlab{}.
\newblock \showarticletitle{TREND: TempoRal Event and Node Dynamics for Graph Representation Learning}. In \bibinfo{booktitle}{\emph{Proceedings of the ACM Web Conference 2022}} (Virtual Event, Lyon, France) \emph{(\bibinfo{series}{WWW '22})}. \bibinfo{publisher}{Association for Computing Machinery}, \bibinfo{address}{New York, NY, USA}, \bibinfo{pages}{1159–1169}.
\newblock
\showISBNx{9781450390965}
\urldef\tempurl%
\url{https://doi.org/10.1145/3485447.3512164}
\showDOI{\tempurl}


\bibitem[Wu et~al\mbox{.}(2019)]%
        {pmlr-v97-wu19e}
\bibfield{author}{\bibinfo{person}{Felix Wu}, \bibinfo{person}{Amauri Souza}, \bibinfo{person}{Tianyi Zhang}, \bibinfo{person}{Christopher Fifty}, \bibinfo{person}{Tao Yu}, {and} \bibinfo{person}{Kilian Weinberger}.} \bibinfo{year}{2019}\natexlab{}.
\newblock \showarticletitle{Simplifying Graph Convolutional Networks}. In \bibinfo{booktitle}{\emph{Proceedings of the 36th International Conference on Machine Learning}} \emph{(\bibinfo{series}{Proceedings of Machine Learning Research}, Vol.~\bibinfo{volume}{97})}, \bibfield{editor}{\bibinfo{person}{Kamalika Chaudhuri} {and} \bibinfo{person}{Ruslan Salakhutdinov}} (Eds.). \bibinfo{publisher}{PMLR}, \bibinfo{pages}{6861--6871}.
\newblock
\urldef\tempurl%
\url{https://proceedings.mlr.press/v97/wu19e.html}
\showURL{%
\tempurl}


\bibitem[Wu et~al\mbox{.}(2020)]%
        {Wu_Lian_Xu_Wu_Chen_2020}
\bibfield{author}{\bibinfo{person}{Yongji Wu}, \bibinfo{person}{Defu Lian}, \bibinfo{person}{Yiheng Xu}, \bibinfo{person}{Le Wu}, {and} \bibinfo{person}{Enhong Chen}.} \bibinfo{year}{2020}\natexlab{}.
\newblock \showarticletitle{Graph Convolutional Networks with Markov Random Field Reasoning for Social Spammer Detection}.
\newblock \bibinfo{journal}{\emph{Proceedings of the AAAI Conference on Artificial Intelligence}} (\bibinfo{date}{Jun} \bibinfo{year}{2020}), \bibinfo{pages}{1054–1061}.
\newblock
\urldef\tempurl%
\url{https://doi.org/10.1609/aaai.v34i01.5455}
\showDOI{\tempurl}


\bibitem[Yang et~al\mbox{.}(2021)]%
        {yang2021discrete}
\bibfield{author}{\bibinfo{person}{Menglin Yang}, \bibinfo{person}{Min Zhou}, \bibinfo{person}{Marcus Kalander}, \bibinfo{person}{Zengfeng Huang}, {and} \bibinfo{person}{Irwin King}.} \bibinfo{year}{2021}\natexlab{}.
\newblock \showarticletitle{Discrete-time temporal network embedding via implicit hierarchical learning in hyperbolic space}. In \bibinfo{booktitle}{\emph{Proceedings of the 27th ACM SIGKDD Conference on Knowledge Discovery \& Data Mining}}. \bibinfo{pages}{1975--1985}.
\newblock


\bibitem[You et~al\mbox{.}(2022)]%
        {You22}
\bibfield{author}{\bibinfo{person}{Jiaxuan You}, \bibinfo{person}{Tianyu Du}, {and} \bibinfo{person}{Jure Leskovec}.} \bibinfo{year}{2022}\natexlab{}.
\newblock \showarticletitle{ROLAND: Graph Learning Framework for Dynamic Graphs}. In \bibinfo{booktitle}{\emph{Proceedings of the 28th ACM SIGKDD Conference on Knowledge Discovery and Data Mining}} (Washington DC, USA) \emph{(\bibinfo{series}{KDD '22})}. \bibinfo{publisher}{Association for Computing Machinery}, \bibinfo{address}{New York, NY, USA}, \bibinfo{pages}{2358–2366}.
\newblock
\showISBNx{9781450393850}
\urldef\tempurl%
\url{https://doi.org/10.1145/3534678.3539300}
\showDOI{\tempurl}


\bibitem[Yu et~al\mbox{.}(2018)]%
        {yu2018spatio}
\bibfield{author}{\bibinfo{person}{Bing Yu}, \bibinfo{person}{Haoteng Yin}, {and} \bibinfo{person}{Zhanxing Zhu}.} \bibinfo{year}{2018}\natexlab{}.
\newblock \showarticletitle{Spatio-temporal graph convolutional networks: a deep learning framework for traffic forecasting}. In \bibinfo{booktitle}{\emph{Proceedings of the 27th International Joint Conference on Artificial Intelligence}}. \bibinfo{pages}{3634--3640}.
\newblock


\bibitem[Zhang et~al\mbox{.}(2023)]%
        {zhang2023dyted}
\bibfield{author}{\bibinfo{person}{Kaike Zhang}, \bibinfo{person}{Qi Cao}, \bibinfo{person}{Gaolin Fang}, \bibinfo{person}{Bingbing Xu}, \bibinfo{person}{Hongjian Zou}, \bibinfo{person}{Huawei Shen}, {and} \bibinfo{person}{Xueqi Cheng}.} \bibinfo{year}{2023}\natexlab{}.
\newblock \showarticletitle{DyTed: Disentangled Representation Learning for Discrete-time Dynamic Graph}. In \bibinfo{booktitle}{\emph{Proceedings of the 29th ACM SIGKDD Conference on Knowledge Discovery and Data Mining}}. \bibinfo{pages}{3309--3320}.
\newblock


\bibitem[Zheng et~al\mbox{.}(2024)]%
        {Zheng2024}
\bibfield{author}{\bibinfo{person}{Yanping Zheng}, \bibinfo{person}{Lu Yi}, {and} \bibinfo{person}{Zhewei Wei}.} \bibinfo{year}{2024}\natexlab{}.
\newblock \showarticletitle{A survey of dynamic graph neural networks}.
\newblock \bibinfo{journal}{\emph{Frontiers of Computer Science}} \bibinfo{volume}{19}, \bibinfo{number}{6} (\bibinfo{date}{12 Dec} \bibinfo{year}{2024}), \bibinfo{pages}{196323}.
\newblock
\showISSN{2095-2236}
\urldef\tempurl%
\url{https://doi.org/10.1007/s11704-024-3853-2}
\showDOI{\tempurl}


\bibitem[Zhu et~al\mbox{.}(2023)]%
        {Zhu23}
\bibfield{author}{\bibinfo{person}{Yifan Zhu}, \bibinfo{person}{Fangpeng Cong}, \bibinfo{person}{Dan Zhang}, \bibinfo{person}{Wenwen Gong}, \bibinfo{person}{Qika Lin}, \bibinfo{person}{Wenzheng Feng}, \bibinfo{person}{Yuxiao Dong}, {and} \bibinfo{person}{Jie Tang}.} \bibinfo{year}{2023}\natexlab{}.
\newblock \showarticletitle{WinGNN: Dynamic Graph Neural Networks with Random Gradient Aggregation Window}. In \bibinfo{booktitle}{\emph{Proceedings of the 29th ACM SIGKDD Conference on Knowledge Discovery and Data Mining}} (, Long Beach, CA, USA,) \emph{(\bibinfo{series}{KDD '23})}. \bibinfo{publisher}{Association for Computing Machinery}, \bibinfo{address}{New York, NY, USA}, \bibinfo{pages}{3650–3662}.
\newblock
\showISBNx{9798400701030}
\urldef\tempurl%
\url{https://doi.org/10.1145/3580305.3599551}
\showDOI{\tempurl}


\bibitem[Zuo et~al\mbox{.}(2018a)]%
        {Zuo2018}
\bibfield{author}{\bibinfo{person}{Yuan Zuo}, \bibinfo{person}{Guannan Liu}, \bibinfo{person}{Hao Lin}, \bibinfo{person}{Jia Guo}, \bibinfo{person}{Xiaoqian Hu}, {and} \bibinfo{person}{Junjie Wu}.} \bibinfo{year}{2018}\natexlab{a}.
\newblock \showarticletitle{Embedding Temporal Network via Neighborhood Formation}. In \bibinfo{booktitle}{\emph{Proceedings of the 24th ACM SIGKDD International Conference on Knowledge Discovery \& Data Mining}} (London, United Kingdom) \emph{(\bibinfo{series}{KDD '18})}. \bibinfo{publisher}{Association for Computing Machinery}, \bibinfo{address}{New York, NY, USA}, \bibinfo{pages}{2857–2866}.
\newblock
\showISBNx{9781450355520}
\urldef\tempurl%
\url{https://doi.org/10.1145/3219819.3220054}
\showDOI{\tempurl}


\bibitem[Zuo et~al\mbox{.}(2018b)]%
        {zuo2018embedding}
\bibfield{author}{\bibinfo{person}{Yuan Zuo}, \bibinfo{person}{Guannan Liu}, \bibinfo{person}{Hao Lin}, \bibinfo{person}{Jia Guo}, \bibinfo{person}{Xiaoqian Hu}, {and} \bibinfo{person}{Junjie Wu}.} \bibinfo{year}{2018}\natexlab{b}.
\newblock \showarticletitle{Embedding temporal network via neighborhood formation}. In \bibinfo{booktitle}{\emph{Proceedings of the 24th ACM SIGKDD international conference on knowledge discovery \& data mining}}. \bibinfo{pages}{2857--2866}.
\newblock


\end{thebibliography}

\appendix

\section{Research Methods}

\subsection{Details of Equation~\ref{eq:decay_denoising}}
In this subsection, we provide details of Equation~\ref{eq:decay_denoising}, proving $tr(\mathbf{F}^T\mathbf{L}\mathbf{F}) = \sum_{(i,j) \in \mathcal{E}}\mathcal{C}_{ij}||\mathbf{F}_i - \mathbf{F}_j||_2^2$, while $\mathbf{L} = diag(\sum_j{{\mathcal{C}}_{ij}}) - {\mathcal{C}}$. For the sake of convenience in symbolic representation, let us assume that $\mathbf{F} \in \mathbb{R}^{n \times 1}$ with elements denoted as $f_1, f_2,\cdots,f_n$, and the process of proving in the high-dimensional case follows the same steps. By using the rule of matrix trace calculation, the left term can be changed into
\[
    \label{eq:trace}
    \begin{split}
        tr(\mathbf{F}^{\top}\mathbf{L}\mathbf{F}) &= tr(\mathbf{F}\mathbf{F}^{\top}\mathbf{L}) \\
        &= tr(\mathbf{F}\mathbf{F}^{\top}diag(\sum_j{{\mathcal{C}}_{ij}})) - tr(\mathbf{F}\mathbf{F}^{\top}\mathcal{C}).
    \end{split}
\]
Where the elements of $\mathbf{F}\mathbf{F}^{\top}$ is
\[
\begin{bmatrix}
    f_1f_1 & f_1f_2 & \cdots  & f_1f_n \\
    f_2f_1 & f_2f_2 & \cdots  & f_2f_n \\
    \vdots & \vdots & \ddots & \vdots \\
    f_nf_1 & f_nf_2 & \cdots  & f_nf_n
\end{bmatrix}.
\]
Since $diag(\sum_j{{\mathcal{C}}_{ij}})$ is a diagonal matrix 
\[
  \begin{bmatrix}
    \sum_j{\mathcal{C}_{1j}} & & \\
    & \ddots & \\
    & & \sum_j{\mathcal{C}_{nj}}
  \end{bmatrix}.
\]
Therefore we have 
\[
  tr(\mathbf{F}\mathbf{F}^{\top}diag(\sum_j{{\mathcal{C}}_{ij}})) = \sum_i\sum_j\mathcal{C}_{ij}f_i^2.
\]
Since the trace operation only sums elements in diagonal, we have 
\[
  tr(\mathbf{F}\mathbf{F}^{\top}{\mathcal{C}}) = \sum_i\sum_j\mathcal{C}_{ij}f_if_j.
\]
Based on the above results, equation $tr(\mathbf{F}\mathbf{F}^{\top}diag(\sum_j{{\mathcal{C}}_{ij}})) - tr(\mathbf{F}\mathbf{F}^{\top}\mathcal{C})$ can be transformed into
\[
  \begin{split}
        &= \sum_i\sum_j\mathcal{C}_{ij}f_i^2 - \sum_i\sum_j\mathcal{C}_{ij}f_if_j \\
        &= \frac{1}{2}(\sum_i\sum_j\mathcal{C}_{ij}f_i^2 - 2\sum_i\sum_j\mathcal{C}_{ij}f_if_j + \sum_i\sum_j\mathcal{C}_{ij}f_j^2) \\
        &= \frac{1}{2}\sum_i\sum_j\mathcal{C}_{ij}(f_i - f_j)^2 \\
        &= \sum_{(i,j) \in \mathcal{E}}\mathcal{C}_{ij}||\mathbf{F}_i - \mathbf{F}_j||_2^2,
  \end{split}
\]
which completes the proof.

\subsection{Complexity Analysis}
\label{sec:complexity}
In this subsection, we provide the detailed descriptions for time and space analysis in Table~\ref{tab:complexity}.

\textbf{Time complexity.} According to the results presented ClusterGCN \cite{chiang2019cluster}, the time complexity of message-passing based GNN is given by $O(lef + lnf^2)$. The DySAT involves obtaining node embeddings for $t$ snapshots within a given time window, which requires a time complexity of $O(tlef + tlnf^2)$. Following this, the self-attention is utilized to encode temporal information for $n$ nodes in $t$ snapshots. The time complexity of the self-attention mechanism is $O(nt^2f)$, while the feature transformation complexity is $O(ntf^2)$. Consequently, the time complexity of DySAT can be expressed as $O(tlef + tlnf^2 + nt^2f)$. Considering usually $lf > t$, then it can be written as $O(tlef + tlnf^2)$. In the case of EvolveGCN, RNN is used to update the parameters at each GNN step in the time window, and RNN performs feature transformation for all nodes in each snapshot, resulting in a time complexity of $O(ntf^2)$. Thus, the total time complexity of EvolveGCN is $O(tlef + tlnf^2)$. In the Roland algorithm, the time window $t$ is set to 1, and no time encoder is utilized, leading to a complexity of $O(lef + lnf^2)$ for a single window. SFDyG merges $t$ snapshots in the time window into a temporal graph, with the number of edges being $te$. Therefore, for the full-batch training of SFDyG, denoted as SFDyG-F, the time complexity is $O(tlef + lnf^2)$. For the mini-batch training version of SFDyG denoted as SFDyG-M, each edge requires feature aggregation from $O(2d^l)$ neighbors, resulting in a time complexity of $O(2ed^lf^2)$.

\textbf{Memory complexity.} Similarly, Based on the conclusions from the ClusterGCN \cite{chiang2019cluster}, the message passing based GNN has a space complexity of $O(lef + lf^2)$. Consequently, the space complexity of the GNN component in DySAT is $O(tlnf + tlf^2)$. The self-attention mechanism requires storing the attention matrix with $O(t^2))$, outputs $(O(tf))$, and feature transformation parameters $O(f^2)$. Therefore, the space complexity of the self-attention component becomes $O(nt^2 + ntf + f^2)$. Typically, with $lf > t$, leading to an overall space complexity of $O(tlnf + tlf^2)$ for DySAT. The RNN part of EvolveGCN needs to store all intermediate results as parameters for GNN, resulting in a space complexity of $O(tnfl + F^2)$. Consequently, the combined space complexity becomes $O(tlnf + tlf^2)$. As Roland has a window size of 1, its space complexity is the same as that of a single GNN, which is $O(lef + lf^2)$. In the case of the full-batch SFDyG-F, the input edge number is $te$, resulting in a space complexity of $O(tlef + lf^2)$. For the mini-batch trained SFDyG-F, a mini-batch can have $bd^L$ edges, resulting in a space complexity of $O(2bd^lf + lf^2)$.

\section{Experiment Details}
\subsection{Description of Datasets}
\label{sec:ap_ds_dsec}
In our experiments, we utilize a combination of synthetic and publicly available benchmark.

\textbf{Stochastic Block Model.}
\footnote{\url{https://github.com/IBM/EvolveGCN}}
(SBM for short)
SBM is a widely employed random graph model utilized for simulating community structures and evolutions. The data utilized in this study was obtained from the GitHub repository of EvolveGCN \cite{egcn20}.

\textbf{Bitcoin OTC.}%
\footnote{\url{http://snap.stanford.edu/data/soc-sign-bitcoin-otc.html}}
(OTC for short)
OTC is a who-trusts-whom network among bitcoin users who engage in trading activities on the platform \url{http://www.bitcoin-otc.com}. The data set may be used for predicting the polarity of each rating and forecasting whether a user will rate another one in the next time step.

\textbf{Bitcoin Alpha.}%
\footnote{\url{http://snap.stanford.edu/data/soc-sign-bitcoin-alpha.html}}
(Alpha for short)
Alpha is created in the same manner as OTC, except that the users and ratings come from a different trading platform, \url{http://www.btc-alpha.com}.

\textbf{UC Irvine messages.}%
\footnote{\url{http://konect.uni-koblenz.de/networks/opsahl-ucsocial}}
(UCI for short)
UCI is an online community of students from the University of California, Irvine, wherein the links of this social network indicate sent messages between users. Link prediction is a standard task for this data set.

\textbf{Autonomous systems.}%
\footnote{\url{http://snap.stanford.edu/data/as-733.html}}
(AS for short)
Autonomous Systems (AS) constitute a communication network of routers that exchange traffic flows with their peers. This dataset can be utilized for predicting future message exchanges.

\textbf{Reddit-Title and Reddit-Body.}%
\footnote{\url{https://snap.stanford.edu/data/soc-RedditHyperlinks.html}}
The network of hyperlinks between subreddits is derived from hyperlinks found in posts. These hyperlinks can appear in either post titles or bodies, resulting in two distinct datasets.

\textbf{Stack Overflow.}%
\footnote{\url{https://snap.stanford.edu/data/sx-stackoverflow.html}}
(SO for short)
The dataset presents interactions within the Stack Overflow platform. Nodes in the dataset represent users, while directed edges indicate the flow of answer activity between users.

\subsection{Description of Baselines}
\label{sec:ap_baseline}

\textbf{DySAT} \cite{sankar2020dysat} computes node representations through joint self-attention along the two dimensions of the structural neighborhood and temporal dynamics.

\textbf{EvolveGCN} \cite{egcn20}: adapts the GCN to compute node representations, and captures the dynamism of the graph sequence by using an RNN to evolve the GCN parameters.

\textbf{ROLAND} \cite{You22}: views the node representations at different GNN layers as hierarchical node states and recurrently updates them.
We present results of the moving average variant of Roland, which can be trained in our GPU environment for the StackOverflow dataset.

\textbf{VGRNN} \cite{hajiramezanali2019variational}: a hierarchical variational model that introduces additional latent random variables to jointly model the hidden states of a graph recurrent neural network (GRNN).

\textbf{HTGN} \cite{yang2021discrete} maps the dynamic graph into hyperbolic space, and incorporates hyperbolic graph neural network and hyperbolic gated recurrent neural network to obtain representations.

\textbf{WinGNN} \cite{Zhu23}: is a GNN model that employs a meta-learning strategy and introduces a novel random gradient aggregation mechanism. 

\textbf{GraphMixer} \cite{congwe}: simplifies temporal graph learning by utilizing MLP-based link and node encoders, achieving high performance in link prediction tasks through faster convergence and improved generalization.

\textbf{M2DNE} \cite{10.1145/3357384.3357943}: is a novel approach for temporal network embedding by effectively capturing both micro- and macro-dynamics through a temporal attention point process and a general dynamics equation parameterized with network embeddings.

\textbf{GHP} \cite{Shang_Sun_2019}: integrates Hawkes processes with a graph convolutional recurrent neural network to efficiently model correlated temporal sequences with improved prediction accuracy.

\subsection{Evaluation Metrics}
\label{sec:ap_metric}
The MRR score is the average of the reciprocal ranks of the positive samples within the negative samples, formally, 
\[
MRR=\frac{1}{N}\sum_{i=1}^{N}{\frac{1}{rank(p_i)}}.
\]

\subsection{Experiment Setup}
\label{sec:ap_setup}
\subsubsection{Running Environment} We perform our comparisons on an Ubuntu 20.04.4 LTS server with an Intel Xeon 32-Core Processor, 200 GB RAM, and an NVIDIA A100-SXM4-80GB Tensor Core GPU. SFDyG is implemented with Python 3.11.4 with PyTorch 2.3.1 and torch\_geometric 2.4.0 framework.

\subsubsection{Hyper-parameter Settings} In all experiments, unless otherwise specified, we ensured consistency by utilizing a standardized set of hyperparameters, with dummy node features (a vector of all ones) to avoid over-fitting, one negative sample for training, and 100 negative samples for testing. Negative samples for testing were generated in advance to maintain reproducibility. The training window size was set to 10, comprising 9 input snapshots and one target snapshot. The default dropout rate was 0.1, the learning rate was 0.001, the hidden layer size was 64 and the patience of early stopping was 20 epochs without MRR increase on the validation dataset. For the mini-batch experiments, $LinkNeighborLoader$ in PyG \cite{Fey/Lenssen/2019} was used to construct the mini-batches.
Additionally, the Adam \cite{kingma2014adam} was used as the optimizer for gradient descent and the Cosine Annealing learning rate scheduler \cite{loshchilov2016sgdr} was used to accelerate training.

\subsubsection{Source Code} The source code is available at \url{https://github.com/oncemoe/hawkesGNN} 

\end{document}